%% file: ptm.tex
\documentclass[11pt]{article}

\usepackage[final]{acl}

\usepackage{times}
\usepackage{latexsym}

\usepackage[T1]{fontenc}

\usepackage[utf8]{inputenc}

\usepackage{microtype}

\usepackage{inconsolata}

\usepackage{graphicx}

\usepackage{graphicx}
\usepackage{xspace}
\newcommand{\method}{\textsc{Prune-then-Merge}\xspace}

\title{Sculpting the Vector Space: Towards Efficient Multi-Vector \\Visual Document Retrieval via \method Framework}

\author{\textbf{Yibo Yan}$^{1,2,3}$, 
    \textbf{Mingdong Ou}$^{2,}$\thanks{Project Lead}, 
    \textbf{Yi Cao}$^{2}$, 
    \textbf{Xin Zou}$^{1,3}$, 
    \textbf{Jiahao Huo}$^{1,2}$,
    \textbf{Shuliang Liu}$^{1,3}$, \\
    \textbf{James Kwok}$^{3}$, 
    \textbf{Xuming Hu}$^{1,3,}$\thanks{Corresponding Author}\\
    $^1$Hong Kong University of Science and Technology (Guangzhou), \\
    $^2$Alibaba Cloud Computing, 
    $^3$Hong Kong University of Science and Technology\\
    \texttt{\href{mailto:yanyibo70@gmail.com}{yanyibo70@gmail.com}},
     \texttt{\href{mailto:xuminghu@hkust-gz.edu.cn}{xuminghu@hkust-gz.edu.cn}}
    \vspace{-3mm}
}

\usepackage[utf8]{inputenc} 
\usepackage[T1]{fontenc}    
\usepackage{url}            
\usepackage{booktabs}       
\usepackage{amsfonts}       
\usepackage{nicefrac}       
\usepackage{microtype}      
\usepackage{xcolor}
\definecolor{bluecite}{HTML}{0071BC}

\input{packages}

\usepackage[toc]{appendix}
\usepackage{etoc}
\usepackage{minitoc}
\etocsettocstyle{\section*{Contents of Technical Appendices}}{}

\begin{document}
\maketitle

\etocdepthtag.toc{mtchapter}
\etocsettagdepth{mtchapter}{subsection}
\etocsettagdepth{mtappendix}{none}

\begin{abstract}
Visual Document Retrieval (VDR), which aims to retrieve relevant pages within vast corpora of visually-rich documents, is of significance in current multimodal retrieval applications. The state-of-the-art multi-vector paradigm excels in performance but suffers from prohibitive overhead, a problem that current efficiency methods like pruning and merging address imperfectly, creating a difficult trade-off between compression rate and feature fidelity. To overcome this dilemma, we introduce \method, a novel two-stage framework that synergizes these complementary approaches. Our method first employs an adaptive pruning stage to filter out low-information patches, creating a refined, high-signal set of embeddings. Subsequently, a hierarchical merging stage compresses this pre-filtered set, effectively summarizing semantic content without the noise-induced feature dilution seen in single-stage methods. Extensive experiments on 29 VDR datasets demonstrate that our framework consistently outperforms existing methods, significantly extending the near-lossless compression range and providing robust performance at high compression ratios. 
\end{abstract}

\input{sections/1-introduction}
\input{sections/2-relatedwork}
\input{sections/3-methodology}

\input{sections/4-experiment}

\input{sections/5-conclusion}

\clearpage
\section*{Limitations}

\begin{itemize}
    \item The efficacy of the pruning stage is inherently tied to the reliability of the base LVLM’s internal attention mechanism as a proxy for patch importance. We plan to investigate more sophisticated, query-independent ranking metrics, such as gradient-based importance, to provide a more accurate signal for information filtering.

    \item The framework currently relies on predefined hyperparameters, such as the adaptation factor and merging factor, to balance the compression-performance trade-off. We will work toward developing more automated, data-driven strategies that can self-adapt these parameters based on document complexity and layout characteristics.

\end{itemize}

\section*{Acknowledgments}
This work was supported by the Alibaba Innovative Research (AIR) Program (Grant No.64575662); Alibaba Research Intern Program; National Natural Science Foundation of China (Grant No.62506318); Guangdong Provincial Department of Education Project (Grant No.2024KQNCX028); Scientific Research Projects for the Higher-educational Institutions (Grant No.2024312096), Education Bureau of Guangzhou Municipality; Guangzhou-HKUST(GZ) Joint Funding Program (Grant No.2025A03J3957), Education Bureau of Guangzhou Municipality.


\bibliography{ptm}

\clearpage
\input{sections/Appendix}




\end{document}

%% file: packages.tex
\usepackage{graphicx}
\usepackage{subcaption}
\usepackage{multirow}

\usepackage{amsmath}
\newtheorem{theorem}{Theorem}[section]

\newtheorem{proof}{Proof Sketch}[section]
\usepackage{paralist}

\newtheorem{corollary}{Corollary}

\usepackage[capitalize]{cleveref}
\crefname{section}{Sec.}{Secs.}
\Crefname{section}{Section}{Sections}
\Crefname{table}{Table}{Tables}
\crefname{table}{Tab.}{Tabs.}

\usepackage{microtype}
\usepackage{booktabs} %

\definecolor{darkgreen}{rgb}{0.0, 0.5, 0.0} 

\newcommand{\comt}[1]{#1}

\renewcommand{\comt}[1]{}

\usepackage[dvipsnames,table]{xcolor}

\usepackage{tabularx}
\usepackage{multicol}
\definecolor{myblue}{RGB}{235,235,250}

\usepackage{xspace}

\usepackage{pifont}
\usepackage{latexsym}
\usepackage{inconsolata}
\usepackage{arydshln}
\usepackage{enumitem}
\usepackage{wrapfig}
\usepackage{array}
\usepackage{epigraph}

\definecolor{lightpink}{RGB}{204, 231, 207} 
\definecolor{lightblue}{RGB}{210, 220, 250} 
\definecolor{lightgray}{RGB}{237, 237, 237} 

\definecolor{superlightred}{rgb}{0.99, 0.92, 0.92}
\definecolor{darkgreen}{RGB}{50,100,0}
\definecolor{darkred}{RGB}{200, 0, 0}

\usepackage{makecell}
\usepackage{colortbl}

\usepackage[ruled,vlined]{algorithm2e}

\usepackage{hyperref}
\usepackage{url}
\usepackage{hyperref}
\usepackage[most]{tcolorbox}
\usepackage{wrapfig}
\usepackage{graphicx,xcolor,float}
\usepackage{subcaption}
\usepackage{threeparttable}
\usepackage{pifont}
\usepackage{colortbl}
\usepackage{color}
\usepackage{multirow}
\usepackage{tabularx}
\usepackage{float}
\usepackage{graphicx}
\usepackage{booktabs}
\usepackage{arydshln}
\usepackage{enumitem}
\usepackage{wrapfig}
\usepackage{caption}
\usepackage{graphicx}
\usepackage{makecell}
\usepackage{float} 
\usepackage{makecell}
\usepackage{tabularx}
\usepackage{twemojis}
\usepackage{amssymb,mathrsfs,amsmath}
\usepackage{pifont}
\usepackage[export]{adjustbox}
\usepackage{xcolor}
\usepackage{xspace}
\usepackage{shadowtext}
\usepackage{anyfontsize}

\hypersetup{
    colorlinks=true,
    linkcolor=red,
    citecolor=cyan,
    filecolor=magenta,      
    urlcolor=magenta,
    }

\definecolor{bittersweet}{rgb}{1.0, 0.44, 0.37}
\definecolor{mygreen}{rgb}{0.29, 0.7, 0.48}
\definecolor{my_green}{RGB}{51,102,0}
\definecolor{my_yellow}{RGB}{255,165,0}
\definecolor{my_red}{RGB}{204, 0, 0}
\definecolor{my_orange}{RGB}{232, 132, 23}

\usepackage{pifont}
\definecolor{mygray}{gray}{0.4}
\hypersetup{
    colorlinks=true,
    linkcolor=red,
    citecolor=cyan,
    filecolor=magenta,      
    urlcolor=magenta,
    }
\usepackage{xcolor} 
\usepackage{xcolor}

\usepackage[font=small,labelfont=bf]{caption}  
\usepackage{makecell}
\usepackage{tabulary}
\definecolor{ada_green}{rgb}{0,205,205}
\definecolor{glt_red}{rgb}{109,205,255}

\definecolor{backred}{RGB}{255, 190, 190}
\definecolor{backblue}{RGB}{210, 230, 250}
\definecolor{backgrey}{RGB}{220, 220, 220}
\definecolor{green(pigment)}{rgb}{0.0, 0.65, 0.31}

\definecolor{shadecolor}{RGB}{237,237,237}



%% file: sections/1-introduction.tex
\section{Introduction}
\label{sec:introduction}

Visual Document Retrieval (VDR) is a critical task that involves retrieving relevant document pages from a vast corpus based on queries that often combine textual and visual cues\footnote{See Appendix \ref{app:illustrative examples} for more illustrative examples.}. 
In an era dominated by visually-rich documents such as reports, slides, and academic papers, VDR has become paramount for applications ranging from enterprise search to domain-specific Retrieval-Augmented Generation (RAG) \cite{zheng2025retrieval,zhao2023retrieving,gao2025scaling}. 
Unlike traditional text retrieval systems that rely on Optical Character Recognition (OCR) to extract textual content \cite{zhang2024ocr}, current approaches leverage Large Vision-Language Models (LVLMs) to process entire document pages as images \cite{jiang2024vlm2vec,meng2025vlm2vec}, as illustrated in Figure \ref{fig:vdr_paradigm_compare} (\textit{Left}). 
This preserves vital structural and layout information, enabling a more holistic, "what-you-see-is-what-you-get" understanding that is crucial for accurately interpreting tables, figures, and complex layouts.

\begin{figure}[t!]
    \centering
    \includegraphics[width=1 \linewidth]{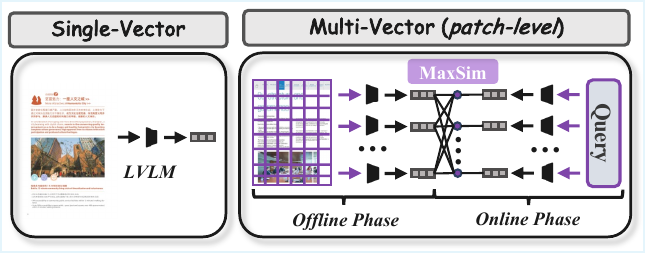}
    \vspace{-2em}
    \caption{Comparison of single-vec vs. multi-vec VDR.}
    \label{fig:vdr_paradigm_compare}
    \vspace{-6mm}
\end{figure}

The field has rapidly evolved from brittle OCR-based pipelines and page-level (single-vector) models to the current state-of-the-art: \textbf{patch-level (multi-vector) retrieval}, as illustrated in Figure \ref{fig:vdr_paradigm_compare} (\textit{Right}). 
Pioneered by seminal works like ColPali \cite{faysse2024colpali}, this paradigm represents each document page as a collection of patch-level embeddings \cite{nomicembedmultimodal2025,teiletche2025modernvbert}.
This fine-grained representation allows for a late-interaction mechanism, such as \texttt{MaxSim}, which computes nuanced, token-to-patch similarities \cite{khattab2020colbert}.
This capability to match a query's specific terms to precise regions within a document has proven to deliver superior retrieval performance, as it captures localized details that single-vector representations often miss.

Despite their exceptional performance, the widespread adoption of multi-vector models is hindered by a critical efficiency bottleneck: \textit{prohibitive storage and computational overhead}  \cite{santhanam2022plaid,chen2024dense,scheerer2025warp}. 
Storing hundreds or even thousands of vectors for every single document page makes large-scale deployment impractical and costly.
Consequently, recent research has focused on optimizing the efficiency, largely converging on two distinct paradigms.
The first is \textbf{pruning-based}, exemplified by DocPruner \cite{yan2025docpruner}, which adaptively discards less informative patches based on intra-document attention, as shown in Figure \ref{fig:vdr_paradigm_compare} (\textit{Left}). This approach can achieve a near-lossless performance with a moderate pruning rate but \textit{suffers from a sharp performance decline at higher compression ratios}.
The second is \textbf{merging-based}, as explored in Light-ColPali/ColQwen series \cite{ma2025towards}, which directly merges multiple patches into fewer vectors, as shown in Figure \ref{fig:vdr_paradigm_compare} (\textit{Right}). This method offers more graceful performance degradation at high compression rates and is simple to implement. However, its crude merging process can dilute discriminative features, leading to an \textit{unstable lossless performance range}.

\begin{figure}[t!]
    \centering
    \includegraphics[width=1 \linewidth]{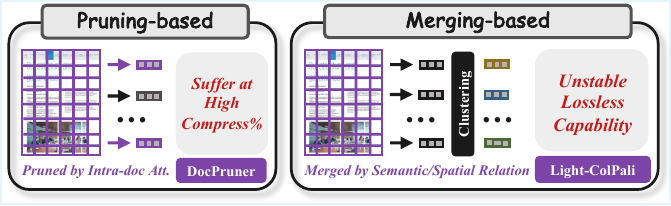}
    \vspace{-2em}
    \caption{Comparison of pruning-based vs. merging-based efficient VDR paradigms.}
    \label{fig:vdr_efficiency_compare}
    \vspace{-6mm}
\end{figure}

To leverage the strengths of these complementary approaches, we propose \textbf{\method}, a novel two-stage framework for efficient multi-vector VDR.
Our core logic is to \textit{first refine, then compress}. 
In the first stage, we employ an adaptive pruning mechanism to intelligently filter out low-information patches, such as empty spaces or decorative elements, leveraging pruning's ability to precisely identify and remove noisy vectors.
In the second stage, we apply hierarchical clustering to merge the remaining set of semantically rich patches.
By performing a more sophisticated merging operation on a pre-filtered, high-quality set of vectors, our method avoids the feature dilution pitfalls of naive merging and pushes beyond the compression limits of standalone pruning.

We conducted extensive experiments to validate our approach on 29 mainstream VDR datasets, integrating \method with three leading multi-vector models: ColQwen2.5 \cite{faysse2024colpali}, ColNomic \cite{nomicembedmultimodal2025}, and Jina-v4 \cite{gunther2025jina}.
Our results yield the following key findings:
\ding{182} Compared to state-of-the-art pruning work (DocPruner), \method extends the near-lossless compression range by an average of 10 percentage points, from [50-60\%] to [60-70\%].
\ding{183} At high compression rates (\textit{e.g.,} 80\% and above), our method circumvents the sharp performance cliff seen in pruning-only methods, consistently outperforming all baselines.

%% file: sections/2-relatedwork.tex
\section{Related Work}
\label{sec:related_work}

\subsection{Visual Document Retrieval}
\label{sec:vdr_paradigms}

The evolution of VDR can be broadly categorized into three successive paradigms. Early approaches were \textbf{OCR-based} ~\citep{xiao2024cpack,wang2023improving,karpukhin2020dense}, relying on OCR tools to extract text from document images, which was then indexed by conventional text retrievers. For a query $q$ and a document image $d$, this can be formulated as scoring based on the extracted text $T(d)$, \textit{i.e.,} $s(q, T(d))$. However, this process is \textit{often brittle and discards critical layout and non-textual information}, limiting its effectiveness on complex documents. The advent of LVLMs catalyzed a shift towards OCR-free methods. The first wave, \textbf{page-level retrieval} (or single-vector retrieval)~\citep{ma2024unifying,zhang2024gme,liu2025any}, encodes an entire document page $d$ and a query $q$ into single, holistic embeddings. Using an encoder $\Phi(\cdot)$, the relevance is typically computed as the similarity between these two vectors: $s(q, d) = \text{sim}(\Phi_q(q), \Phi_d(d))$
where $\Phi_q(q) \in \mathbb{R}^{D}$ and $\Phi_d(d) \in \mathbb{R}^{D}$. While this paradigm preserves the document's visual integrity, the single-vector representation often \textit{fails to capture fine-grained details necessary for precise matching}.

To overcome this, the current paradigm has converged on \textbf{patch-level retrieval} (or multi-vector retrieval)~\citep{masry2024colflor,nomicembedmultimodal2025,gunther2025jina,masry2025colmate}. Pioneered by ColPali~\cite{faysse2024colpali}, this approach represents a document page as a set of multiple patch-level embeddings $\mathbf{D} = \{\mathbf{d}_j\}_{j=1}^{N_p}$, where $\mathbf{d}_j \in \mathbb{R}^{D}$, and the query as a set of token-level embeddings $\mathbf{Q} = \{\mathbf{q}_i\}_{i=1}^{N_q}$. Relevance is then computed via a late-interaction mechanism like \texttt{MaxSim}:
\begin{equation}
    s(q, d) = \sum_{i=1}^{N_q} \max_{j=1}^{N_p} \mathbf{q}_i^\top \mathbf{d}_j .
    \label{eq:maxsim}
\end{equation}

\subsection{Efficient VDR}
\label{sec:efficient_vdr}

Multi-vector VDR models are hampered by a significant efficiency bottleneck: prohibitive storage overhead. Storing a set of $N_p$ embeddings for each document page results in a storage cost of $O(N_p \times D)$ per page, which is orders of magnitude greater than the $O(D)$ cost of single-vector models. This challenge has spurred research into two primary efficiency optimization paradigms.

The first is \textbf{pruning-based} strategies, which aim to discard redundant or less informative patch embeddings during the offline indexing stage~\cite{zong2025towards,lassance2022learned}. DocPruner~\cite{yan2025docpruner} is a representative work in this area, introducing an adaptive mechanism that leverages intra-document attention to identify and remove patches. However, pruning methods often face \textit{a sharp performance drop at higher rates}. The second paradigm is \textbf{merging-based} strategies, which reduce the number of embeddings by combining multiple patches into a smaller set of vectors~\cite{clavie2024reducing,macavaney2025efficient}. Light-ColPali~\cite{ma2025towards} explores this by applying techniques like spatial pooling or semantic clustering. However, the averaging process inherent in merging can \textit{dilute the distinctive features of highly salient patches}, often leading to a less stable performance. Furthermore, recent MetaEmbed~\cite{xiao2025metaembed} compresses tokens via a budget-based compression rate (\textit{e.g.,} Matryoshka embeddings), but it requires well-designed training and architecture change.
Therefore, we propose \method, a hybrid framework that synergizes these two paradigms to overcome their individual limitations. 

See more related works in Appendix \ref{app:more_related_work}.

%% file: sections/3-methodology.tex
\section{Methodology}
\label{sec:methodology}

In this section, we first formalize the task setting of multi-vector VDR. We then introduce \method, detailing its two-stage mechanism.

\subsection{Task Setting}
\label{sec:task_setting}

The task of VDR is to retrieve a ranked list of relevant document pages from a corpus $\mathcal{C} = \{d_1, d_2, \dots, d_{|\mathcal{C}|}\}$ for a given textual query $q$. In the multi-vector retrieval paradigm, both queries and documents are represented as sets of embedding vectors. A LVLM-based encoder, denoted as $\Phi(\cdot)$, maps a textual query $q$ into a set of $N_q$ token-level embeddings $\mathbf{Q} = \{\mathbf{q}_i\}_{i=1}^{N_q}$, where each $\mathbf{q}_i \in \mathbb{R}^{D}$ and $D$ is the embedding dimension. Similarly, a document page $d$, rendered as an image, is processed by the same encoder $\Phi(\cdot)$ to produce a set of $N_p$ patch-level embeddings $\mathbf{D} = \{\mathbf{d}_j\}_{j=1}^{N_p}$, where each $\mathbf{d}_j \in \mathbb{R}^{D}$.

Following the late-interaction mechanism, the relevance score $s(q, d)$ between the query and the document is computed via the \texttt{MaxSim} operation as defined in Equation~\eqref{eq:maxsim}. The primary challenge addressed in this work is the storage overhead of $O(N_p \times D)$ per page. Our objective is to generate a compressed set of document embeddings $\mathbf{D}''$ with size $N_p'' = |\mathbf{D}''|$, such that $N_p'' \ll N_p$, thereby substantially reducing storage costs while minimizing degradation in retrieval performance.

\subsection{The \method~Framework}
\label{sec:our_framework}

\method~is a \textbf{query-agnostic, offline compression} framework designed to synergize the precision of pruning with the high-ratio compression capability of merging. The framework operates in two sequential stages: (1) an adaptive pruning stage to filter out low-information patches, followed by (2) a hierarchical merging stage to compress the remaining semantically rich patches. See our pseudo-code in Appendix \ref{app:algo_workflow} for reference.

\subsubsection{Stage 1: Adaptive Pruning}

The first stage aims to create an intermediate, refined set of patch embeddings by discarding those with low informational content, such as background or decorative elements. This is achieved by leveraging the LVLM's internal attention mechanism as a proxy for patch importance.

Formally, during the forward pass of a document image $d$ through the encoder $\Phi(\cdot)$, we extract the attention weights $\mathbf{A}^{(L)} \in \mathbb{R}^{H \times S \times S}$ from the final transformer layer, where $H$ is the number of attention heads and $S$ is the sequence length. We average the weights across all heads to obtain a smoothed attention map $\bar{\mathbf{A}}^{(L)} \in \mathbb{R}^{S \times S}$. The importance score $I(\mathbf{d}_j)$ for the $j$-th patch is defined as the attention it receives from a global token (\textit{e.g.,} the \texttt{[EOS]} token): $I(\mathbf{d}_j) = \bar{\mathbf{A}}^{(L)}_{\text{eos}, j}$.
This yields a vector of importance scores $\mathcal{I}_d = \{I(\mathbf{d}_j)\}_{j=1}^{N_p}$.

We then compute a document-specific adaptive threshold $\tau_d$ based on the statistical properties of these scores: $\mu_d = \frac{1}{N_p} \sum_{j=1}^{N_p} I(\mathbf{d}_j) \quad$; $\quad \sigma_d = \sqrt{\frac{1}{N_p} \sum_{j=1}^{N_p} (I(\mathbf{d}_j) - \mu_d)^2}$; $\tau_d = \mu_d + k \cdot \sigma_d$, where $k$ is a hyperparameter controlling the pruning strictness. The intermediate set of pruned embeddings, $\mathbf{D}'$, is formed by retaining only patches whose importance exceeds this threshold: $\mathbf{D}' = \{\mathbf{d}_j \in \mathbf{D} \mid I(\mathbf{d}_j) > \tau_d\}$.
To ensure robustness, if this process yields an empty set, we retain the single patch with the highest importance score. The resulting set $\mathbf{D}'$ contains $N_p' = |\mathbf{D}'|$ embeddings, where $N_p' < N_p$.

\subsubsection{Stage 2: Hierarchical Merging}

The second stage takes the filtered set of high-quality embeddings $\mathbf{D}'$ and further compresses it through semantic clustering. This avoids the feature dilution that occurs when merging is applied to the full, noisy set of patches.

Given the intermediate set $\mathbf{D}' \subset \mathbb{R}^{N_p' \times D}$ and a merging factor $m$, we first determine the target number of clusters $N_p''$: $N_p'' = \max(1, \lfloor N_p' / m \rfloor)$.
We then perform hierarchical agglomerative clustering on $\mathbf{D}'$. The process involves:
\begin{enumerate}
    \item \textbf{Normalization:} All embeddings in $\mathbf{D}'$ are L2-normalized.
    \item \textbf{Distance Matrix:} A pairwise distance matrix $\mathbf{\Delta} \in \mathbb{R}^{N_p' \times N_p'}$ is computed based on cosine distance between the normalized embeddings.
    \item \textbf{Clustering:} A linkage algorithm (\textit{e.g.,} Ward's method) is applied to $\mathbf{\Delta}$ to build a hierarchy, which is then partitioned into $N_p''$ clusters. This yields a cluster assignment label for each embedding in $\mathbf{D}'$.
\end{enumerate}
For each of the $N_p''$ clusters, a new representative embedding is generated by computing the centroid (mean) of all member embeddings. This results in the final, compressed set of embeddings $\mathbf{D}'' = \{\mathbf{d}''_c\}_{c=1}^{N_p''}$, where each centroid $\mathbf{d}''_c \in \mathbb{R}^D$. This two-stage process effectively reduces the initial $N_p$ embeddings to a much smaller set of $N_p''$ semantically rich representations.

\subsubsection{Scoring with Pruned-and-Merged Embeddings}
During online retrieval, the relevance score is computed using the final compressed embedding set $\mathbf{D}''$. The query $q$ is encoded into $\mathbf{Q}$ as usual, and the \texttt{MaxSim} score is calculated over the reduced search space: $s''(q, d) = \sum_{i=1}^{N_q} \max_{c=1}^{N_p''} \mathbf{q}_i^\top \mathbf{d}''_c.$
By performing the expensive \texttt{MaxSim} operation on a significantly smaller set of embeddings ($N_p'' \ll N_p$), our framework drastically reduces online computational costs and offline storage requirements, while preserving high retrieval accuracy.

\subsection{Theoretical Guarantee}
\label{sec:theoretical_guarantee}

The empirical success of the \method~framework is underpinned by a sound theoretical basis. We posit that its superiority stems from decomposing a single, complex compression problem into a sequence of two more tractable, specialized sub-problems: (1) query-agnostic information filtering via pruning, and (2) redundancy reduction via merging. This decomposition allows for a more effective approximation of the ideal, yet intractable, Information Bottleneck (IB) objective.

\subsubsection{The Optimization Problem as an IB}

The overarching goal is to compress the full patch embedding set $\mathbf{D}$ into a compact set $\mathbf{D}''$ that retains maximum relevance to any potential query. Following the IB principle, this can be formulated as: $\max_{\mathbf{D}''} \quad \mathbb{E}_{q \sim P(q)}[I(\mathbf{D}''; s(q, \mathbf{D}))] - \beta I(\mathbf{D}''; \mathbf{D})
    \label{eq:ib_ideal}$
where $s(q, \mathbf{D})$ is the relevance score and $P(q)$ is the unknown query distribution. This problem is intractable. Our framework approximates this by decomposing the overall compression mapping $g: \mathbf{D} \to \mathbf{D}''$ into a composition of two functions, $g = g_m \circ g_p$, where $g_p$ is the pruning map and $g_m$ is the merging map.

\subsubsection{Pruning as Query-Agnostic Information Filtering}

The first stage, pruning, acts as an information filter. We assume the full patch set $\mathbf{D}$ can be partitioned into a high-information signal set $\mathbf{D}_{\text{sig}}$ and a low-information noise set $\mathbf{D}_{\text{noi}}$. The key assumption is that for any patch $\mathbf{d}_{\text{noi}} \in \mathbf{D}_{\text{noi}}$, its contribution to the document's global semantics (represented by the \texttt{[EOS]} token's hidden state $\mathbf{h}_{\text{eos}}$) is negligible: $I(\mathbf{d}_{\text{noi}}; \mathbf{h}_{\text{eos}}) \approx 0, \quad \forall \mathbf{d}_{\text{noi}} \in \mathbf{D}_{\text{noi}}$.
The objective of the pruning stage $g_p$ is to produce an intermediate set $\mathbf{D}'$ that approximates the true signal set, $\mathbf{D}' \approx \mathbf{D}_{\text{sig}}$, by maximizing the preserved information about the document's global meaning: $\mathbf{D}' = g_p(\mathbf{D}) = \arg\max_{\hat{\mathbf{D}} \subset \mathbf{D}} \quad I(\hat{\mathbf{D}}; \mathbf{h}_{\text{eos}})$,
subject to a compression constraint. Our adaptive thresholding mechanism, which leverages the attention scores $I(\mathbf{d}_j) \propto I(\mathbf{d}_j; \mathbf{h}_{\text{eos}})$, serves as a practical solver for this objective. This filtering step is crucial as it transforms the input for the next stage from a low Signal-to-Noise Ratio (SNR) set $\mathbf{D}$ to a high-SNR set $\mathbf{D}'$.

\subsubsection{Merging as Rate-Distortion Optimization}

The second stage, merging, addresses the semantic redundancy within the filtered set $\mathbf{D}'$. This is a classic vector quantization problem, which can be framed under Rate-Distortion theory. The goal is to find a "codebook" $\mathbf{D}''$ of size $N_p''$ (the "rate") that minimizes the information loss, or "distortion," with respect to the high-SNR signal $\mathbf{D}'$. The distortion is quantified by Mean Squared Error (MSE). The objective of merging map $g_m$ is: $\mathbf{D}'' = g_m(\mathbf{D}') = \underset{\hat{\mathbf{D}}'' \subset \mathbb{R}^D, |\hat{\mathbf{D}}''|=N_p''}{\arg\min} \mathbb{E}_{\mathbf{d}_j \in \mathbf{D}'} \left[ \min_{\hat{\mathbf{d}} \in \hat{\mathbf{D}}''} ||\mathbf{d}_j - \hat{\mathbf{d}}||^2_2 \right]$.
Hierarchical clustering followed by centroid computation is a highly effective algorithm for solving this. For a given partition of $\mathbf{D}'$ into clusters $\{C_c\}_{c=1}^{N_p''}$, the optimal centroid $\mathbf{d}''_c$ for each cluster that minimizes the intra-cluster variance is its mean: $\mathbf{d}''_c = \underset{\mathbf{v} \in \mathbb{R}^D}{\arg\min} \sum_{\mathbf{d}_j \in C_c} ||\mathbf{d}_j - \mathbf{v}||^2_2 = \frac{1}{|C_c|} \sum_{\mathbf{d}_j \in C_c} \mathbf{d}_j$.
This stage efficiently reduces redundancy by creating a compact, summary representation of the core semantic concepts.

\subsubsection{The Synergistic Advantage}
The efficacy of \method~arises from the synergy between these two stages. Let $\mathbf{d}^*_c(\mathcal{S})$ denote the optimal centroid for a cluster within a set $\mathcal{S}$. A naive, single-stage merging approach operates directly on the noisy set $\mathbf{D} = \mathbf{D}_{\text{sig}} \cup \mathbf{D}_{\text{noi}}$. Its resulting centroids are inherently biased estimators of the true semantic concepts:
\begin{equation}
    \mathbf{d}^*_c(\mathbf{D}) = \frac{1}{|C_c|} \left( \sum_{\mathbf{d}_j \in C_c \cap \mathbf{D}_{\text{sig}}} \mathbf{d}_j + \sum_{\mathbf{d}_k \in C_c \cap \mathbf{D}_{\text{noi}}} \mathbf{d}_k \right).
\end{equation}
The noise vectors $\mathbf{d}_k$ pull the centroids away from the true signal's center of mass. In contrast, our framework first isolates $\mathbf{D}' \approx \mathbf{D}_{\text{sig}}$, so the subsequent merging stage computes centroids $\mathbf{d}^*_c(\mathbf{D}')$ that are largely unbiased by noise. Consequently, the distortion of the final representation with respect to the \textit{true signal} is significantly lower for our method. Let $\mathbf{D}''_{\text{ours}} = g_m(g_p(\mathbf{D}))$ and $\mathbf{D}''_{\text{naive}} = g_m(\mathbf{D})$. We can assert: $\mathbb{E}_{\mathbf{d}_j \in \mathbf{D}_{\text{sig}}} \left[ \min_{\hat{\mathbf{d}} \in \mathbf{D}''_{\text{ours}}} ||\mathbf{d}_j - \hat{\mathbf{d}}||^2_2 \right] \ll \mathbb{E}_{\mathbf{d}_j \in \mathbf{D}_{\text{sig}}} \left[ \min_{\hat{\mathbf{d}} \in \mathbf{D}''_{\text{naive}}} ||\mathbf{d}_j - \hat{\mathbf{d}}||^2_2 \right]$.
By decomposing the problem, \method~ensures that the final compact representation is a more faithful summary of the document's essential information, leading to a superior performance-compression trade-off.

See more theoretical analysis in Appendix \ref{app:more_theory}.

%% file: sections/4-experiment.tex
\section{Experiments and Analysis}
\label{sec:experiment}

\subsection{Experimental Setup}
\label{sec:experimental_setup}

\paragraph{Benchmarks \& Evaluation.}
Our experimental validation is performed on a comprehensive suite of six representative VDR benchmarks, totaling 29 distinct datasets (more details in Appendix \ref{app:benchmark_details}). These include \textbf{ViDoRe-V1}~\citep{faysse2024colpali}, \textbf{ViDoRe-V2}~\citep{macé2025vidorebenchmarkv2raising}, \textbf{JinaVDR-Bench}~\citep{gunther2025jina}, \textbf{REAL-MM-RAG}~\citep{wasserman2025realmmragrealworldmultimodalretrieval}, \textbf{ViDoSeek}~\citep{wang2025vidorag}, and \textbf{MMLongBench-Doc}~\citep{ma2024mmlongbench}. We integrate our framework with three leading multi-vector models to serve as our base models: \textbf{ColQwen2.5}~\citep{faysse2024colpali}, \textbf{ColNomic}~\citep{nomicembedmultimodal2025}, and \textbf{Jina Embeddings v4}~\citep{gunther2025jina}. Following standard VDR practices~\citep{faysse2024colpali,gunther2025jina,xu2025llama}, we adopt \textbf{nDCG@5} as the primary evaluation metric.

\paragraph{Baselines.}
We benchmark our \method{} against three categories of baselines. See detailed elaboration of baselines in Appendix \ref{app:baseline_details}.

\begin{figure*}[!t]
\centering
\includegraphics[width=\linewidth]{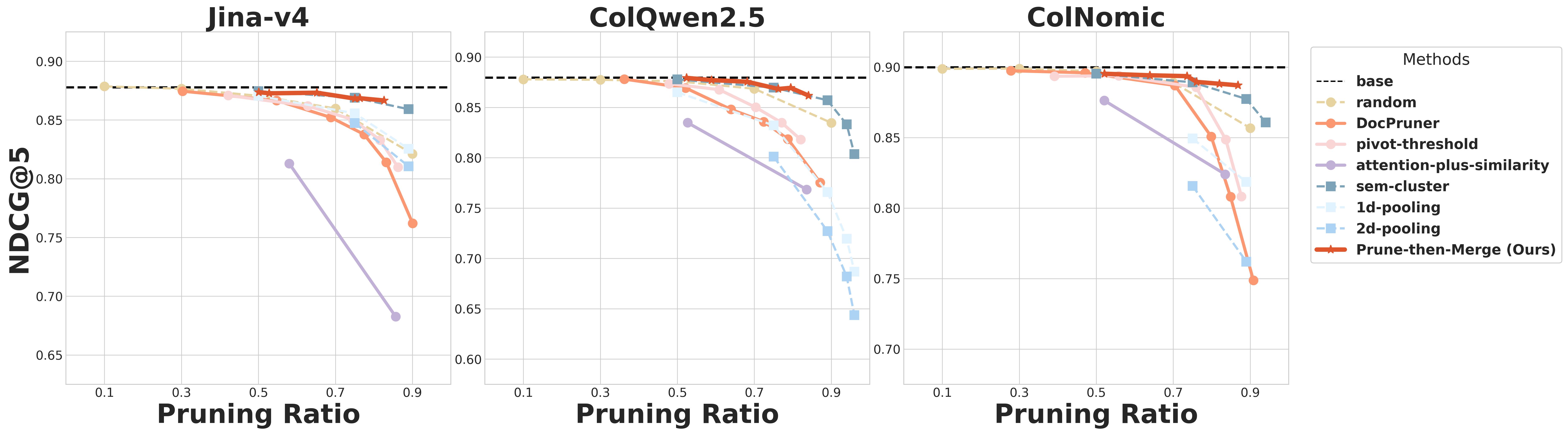}
\vspace{-2em}
\caption{Performance comparison (nDCG@5) between \method and baselines on ViDoRe-V1 \citep{faysse2024colpali} across  Jina-v4 (\textbf{\textit{Left}}), ColQwen2.5 (\textbf{\textit{Middle}}), and ColNomic (\textbf{\textit{Right}}). \textit{solid lines} denote adaptive methods, whereas \textit{dashed lines} denote non-adaptive ones; \textit{circular nodes} represent pruning methods, whereas \textit{square nodes }represent merging ones.} 
\label{fig:vidore_v1_performance_all}
\vspace{-1.3em}
\end{figure*}

\textbf{(I) Base Models.} These are the original, uncompressed multi-vector models. They serve as the performance ceiling and storage-cost upper bound.

\textbf{(II) Pruning-based Methods.} We compare against four strategies that reduce vector count:
\begin{itemize}[leftmargin=1em,itemsep=-0.1em]
    \item[$\blacktriangleright$] \textbf{Random:} A naive baseline that randomly removes a fixed fraction of patch embeddings. Its hyperparameter is the \texttt{pruning\_ratio}.
    \item[$\blacktriangleright$] \textbf{Attention-plus-Similarity:} An adaptive approach that prunes patches based on a combined score of importance (\texttt{[EOS]} attention) and representativeness (\texttt{[EOS]} similarity), following~\citet{wen2025token}. Key hyperparameters are an adaptation factor \texttt{k} and a weighting factor \texttt{alpha}.
    \item[$\blacktriangleright$] \textbf{Pivot-Threshold:} A two-stage adaptive method inspired by VisPruner~\citep{zhang2025vispruner}, which first filters patches via an adaptive attention threshold and then de-duplicates the resulting set based on similarity to pivot tokens. Hyperparameters include an importance factor \texttt{k}, a de-duplication factor \texttt{k\_dup}, and \texttt{num\_pivots}.
    \item[$\blacktriangleright$] \textbf{DocPruner}: A state-of-the-art adaptive method that discards less informative patches based on the intra-document attention distribution, following ~\citet{yan2025docpruner}. Its primary hyperparameter is the adaptation factor \texttt{k}.
\end{itemize}

\textbf{(III) Merging-based Methods.} Following~\citet{ma2025towards}, we compare three merging strategies:
\begin{itemize}[leftmargin=1em,itemsep=-0.1em]
    \item[$\blacktriangleright$] \textbf{Sem-Cluster:} Applies hierarchical clustering to patch embeddings and uses cluster centroids as the merged representation. The tunable hyperparameter is the \texttt{merging\_factor}.
    \item[$\blacktriangleright$] \textbf{1D-Pooling:} Groups sequential patch embeddings and reduces them via 1D average pooling. The hyperparameter is the \texttt{merging\_factor}, defining the pooling window size.
    \item[$\blacktriangleright$] \textbf{2D-Pooling:} Organizes embeddings into a 2D grid and applies 2D average pooling. The \texttt{merging\_factor} must be a perfect square.
\end{itemize}

\paragraph{Implementation Details.}
To ensure fair comparisons, we reproduced the results for all base models in alignment with their official implementations. Our evaluation framework is built upon the official ViDoRe Benchmark repository\footnote{\url{https://github.com/illuin-tech/vidore-benchmark}}. For \method, we explore a range of hyperparameters: the adaptation factor $k$ for the pruning stage is selected from $\{-1, -0.75, -0.5\}$, and the merging factor $m$ for the merging stage is chosen from $\{2, 4\}$. Detailed hyperparameter settings for all baselines are provided in Appendix \ref{app:baseline_details}. All experiments were conducted on a cluster of NVIDIA A100 (80GB) GPUs. The complete codebased will be made publicly available upon acceptance.

\subsection{Experimental Analysis}
\label{sec:experimental_analysis}

In this section, we conduct a comprehensive experimental analysis to rigorously evaluate \method framework\footnote{See all empirical experimental results in Appendix \ref{app:empirical_exp_results}.}, which is guided by five research questions (RQs): 
\textbf{(RQ1)} How does \method~maintain its superior performance across a wide variety of visual document types? 
\textbf{(RQ2)} Does the performance advantage of \method~generalize robustly to multilingual retrieval scenarios?
\textbf{(RQ3)} Can the framework's effectiveness extend to more complex, real-world settings that require semantic understanding instead of keyword matching?
\textbf{(RQ4)} What is the performance gap between \method~framework and its variants?
\textbf{(RQ5)} What are the quantifiable gains in storage efficiency that \method~achieves?

\subsubsection{VDR Performance Comparison}
\label{sec:rq1_perf_diverse_types}

\begin{figure*}[!t]
\centering
\includegraphics[width=\linewidth]{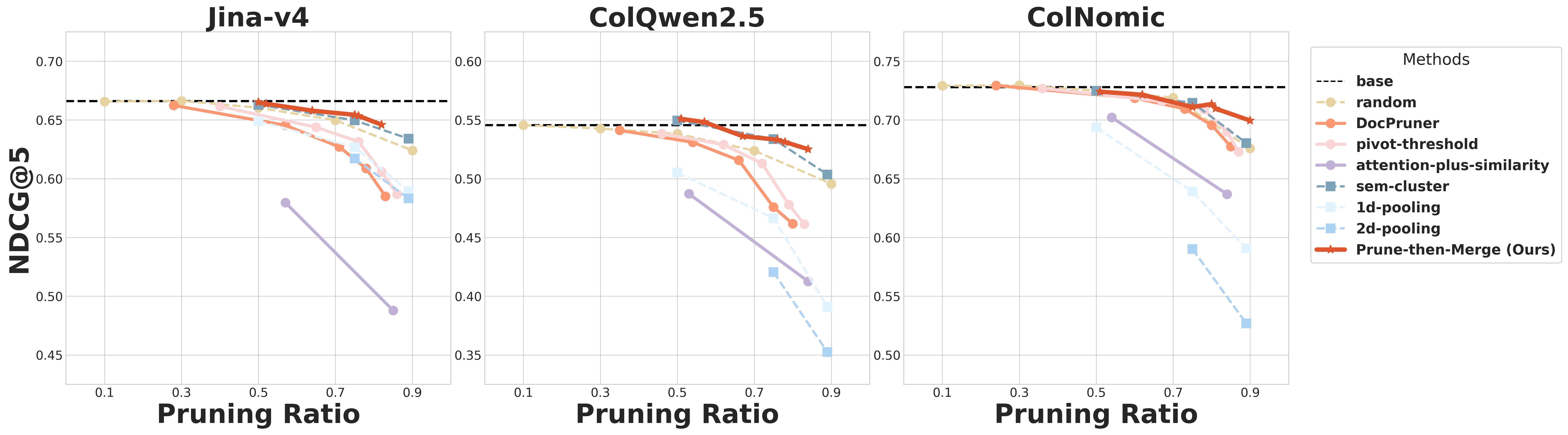}
\vspace{-2em}
\caption{Performance comparison (nDCG@5) between \method and baselines on JinaVDR \citep{gunther2025jina} across  Jina-v4 (\textbf{\textit{Left}}), ColQwen2.5 (\textbf{\textit{Middle}}), and ColNomic (\textbf{\textit{Right}}). \textit{solid lines} denote adaptive methods, whereas \textit{dashed lines} denote non-adaptive ones; \textit{circular nodes} represent pruning methods, whereas \textit{square nodes }represent merging ones.}
\label{fig:jinavdr_performance_all}
\end{figure*}

\begin{figure*}[!t]
\centering
\includegraphics[width=\linewidth]{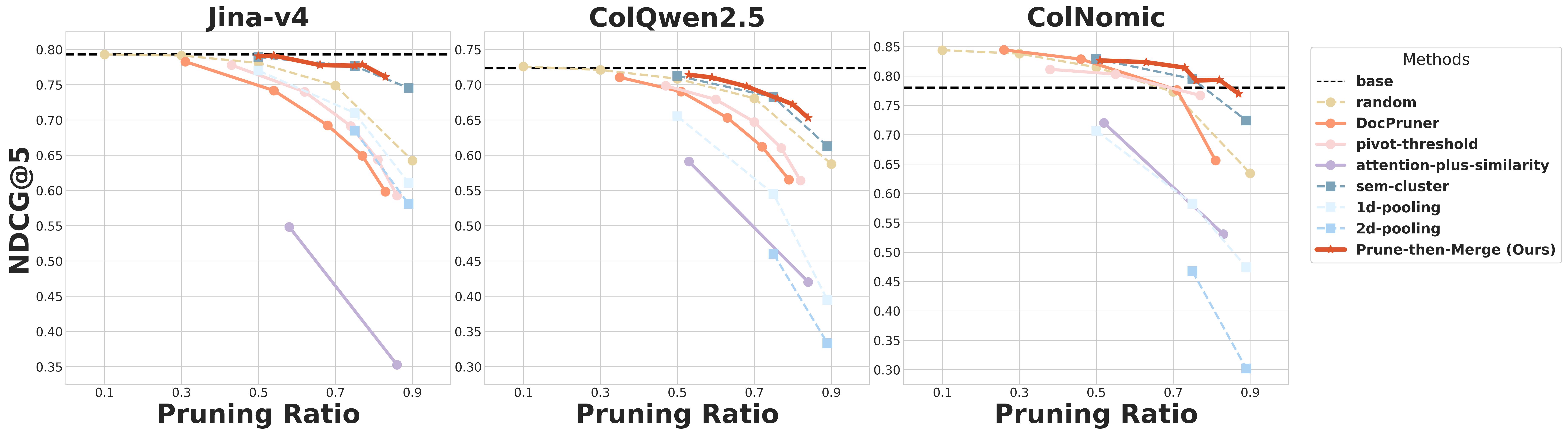}
\vspace{-2em}
\caption{Performance comparison (nDCG@5) between \method and baselines on REAL-MM-RAG \citep{wasserman2025realmmragrealworldmultimodalretrieval} across  Jina-v4 (\textbf{\textit{Left}}), ColQwen2.5 (\textbf{\textit{Middle}}) \& ColNomic (\textbf{\textit{Right}}). \textit{solid lines} denote adaptive methods, whereas \textit{dashed lines} denote non-adaptive ones; \textit{circular nodes} represent pruning methods, whereas \textit{square nodes }represent merging ones.}
\label{fig:realmmrag_performance_all}
\vspace{-1.3em}
\end{figure*}

Our \method framework consistently demonstrates near-lossless performance at high compression rates, maintaining base performance up to approximately a 70\% pruning rate across a diverse array of 16 datasets from four major VDR benchmarks (We only leave ViDoRe-V1 illustration in main text due to space limit. See the rest in Appendix \ref{app:vdr_performance_comparison}.).
For instance, on the comprehensive ViDoRe-V1 (Figure \ref{fig:vidore_v1_performance_all}), our method maintains the base nDCG@5 of 0.87 with ColQwen2.5, even after compressing the embeddings by 68\%.

At extremely high compression rates (80-90\%), where the performance of baselines typically collapses, \method consistently establishes its superiority over both pruning- and merging-only baselines.
Pruning-only methods like DocPruner suffer a sharp performance cliff beyond a 70-80\% pruning rate.
For example, on ViDoRe-V1 with ColQwen2.5 at an \textasciitilde84-87\% pruning rate, \method achieves an nDCG@5 of 0.86, whereas DocPruner drops sharply to 0.77.
While merging-based methods like Sem-Cluster exhibit more graceful degradation, our framework still frequently matches or outperforms them.
This advantage stems from our method’s ability to avoid two failure modes: it circumvents the drastic information loss of aggressive pruning by first identifying and then summarizing semantic clusters, and it avoids feature dilution of naive merging by operating on a pre-filtered, high-signal set of embeddings.

A notable observation is that \method appears to have a slightly extended near-lossless compression range when integrated with the Jina-v4 model.
For example, on both the ViDoSeek and MMLongBench-Doc, our framework maintains the full baseline performance of the Jina-v4 model up to a 75\% pruning rate.
In contrast, a pure pruning approach like DocPruner experiences a significant performance drop at a similar compression level on these datasets (\textit{e.g.,} from a baseline of 0.54 to 0.48 on MMLongBench-Doc at a 77\% rate).
We speculate this is linked to Jina-v4's unique training paradigm, which simultaneously optimizes for both single-vector (pooled) and multi-vector representations \cite{gunther2025jina}.
This inherent training for pooling likely makes its embeddings more amenable to merging operations.

\subsubsection{Generalization to Multilingual Scenarios}
\label{sec:rq2_perf_multilingual}

Our framework's effectiveness generalizes robustly across the nine diverse languages in the JinaVDR, proving its utility in global-scale retrieval systems.
The results, shown in Figure~\ref{fig:jinavdr_performance_all}, demonstrate that \method consistently maintains a superior performance-compression trade-off regardless of the language.
For example, when paired with ColQwen2.5, our method achieves an overall nDCG@5 of 0.52 at an 84\% compression rate, outperforming DocPruner, which scores 0.46 at a lower 80\% rate.
This language-agnostic capability stems from our reliance on the VLM's universal, pre-trained understanding of visual and semantic structures, allowing the framework to identify and compress information without being biased by language-specific features. Due to space limit, see more discussion in Appendix \ref{app:generalization_multilingual}.

\subsubsection{Generalization to Complex Settings}
\label{sec:rq3_perf_complex_settings}

The superiority of \method extends to the challenging REAL-MM-RAG benchmark, which is designed to test deep semantic understanding through non-extractive, rephrased queries.
Across all three base models, our framework consistently outperforms baselines, especially at high compression rates, as shown in Figure~\ref{fig:realmmrag_performance_all}.
With the ColQwen2.5 model, for instance, \method achieves an nDCG@5 of 0.65 at an aggressive 84\% compression rate, clearly surpassing DocPruner (0.56) and Sem-Cluster (0.61) at similar or lower rates.
This proves that by distilling documents into a set of core semantic centroids, our method creates a representation that is robust to abstract, rephrased queries.

Our framework's advantage is especially critical for dense, text-heavy document formats, where pruning-only methods falter dramatically under high compression.
On the dense \texttt{Financial Report} dataset with Jina-v4, DocPruner's performance plummets from a 0.69 baseline to just 0.44 at an 84\% pruning rate, marking a 36\% relative drop.
In stark contrast, \method gracefully degrades to only 0.66 at a comparable 83\% rate, demonstrating remarkable stability.
This suggests that dense documents contain high levels of semantic redundancy that are poorly handled by aggressive pruning; our method's merging stage excels in this scenario by summarizing semantically related text chunks into robust centroids, preserving the document's holistic meaning far more effectively. See more discussion in Appendix \ref{app:generalization_complex setting}.

\subsubsection{Variant Analysis}
\label{sec:rq4_variant_analysis}

\vspace{-2mm}
\begin{figure}[h!]
    \centering
    \includegraphics[width=0.7 \linewidth]{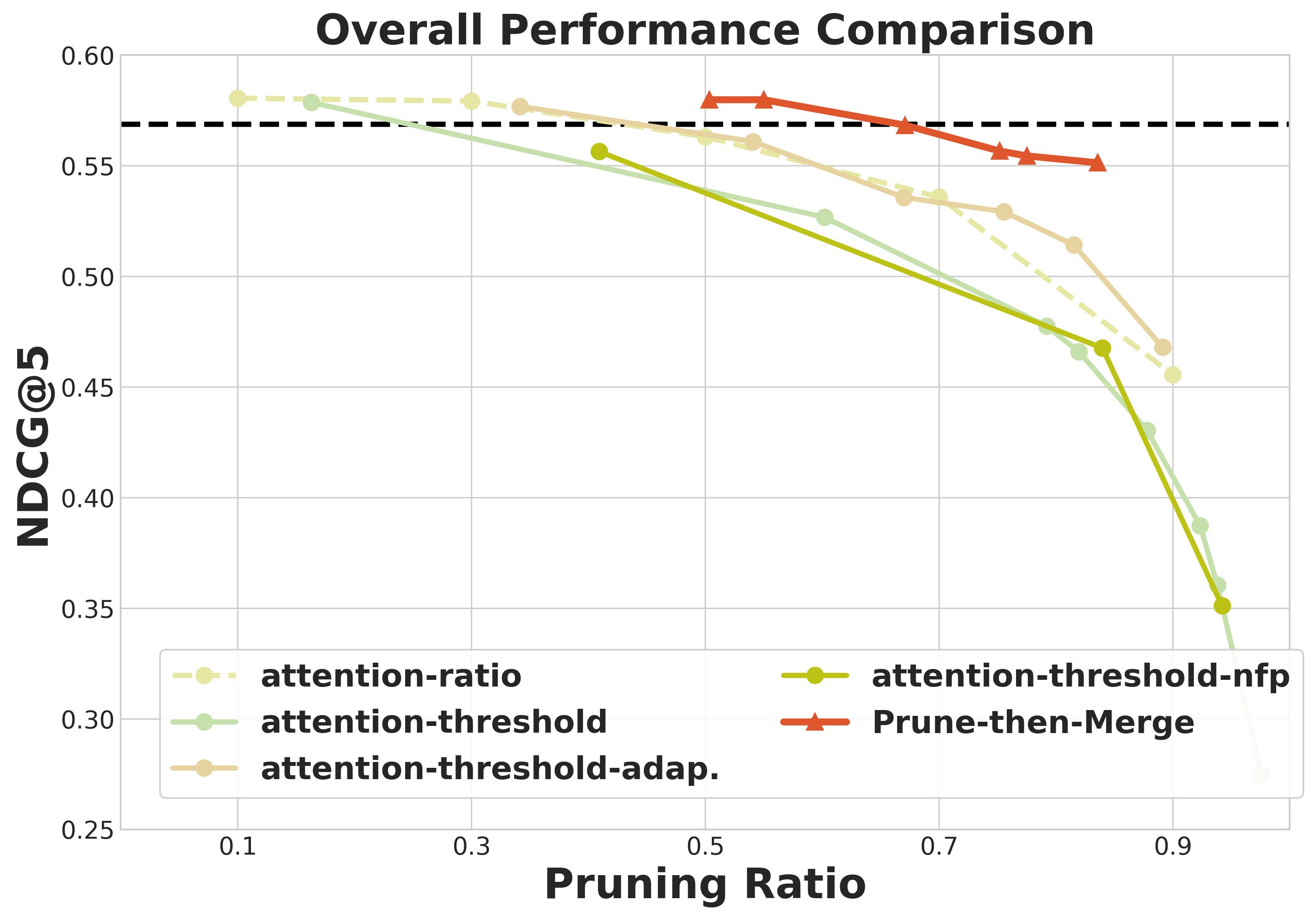}
    \caption{Variant comparisons of Jina-v4 on ViDoRe-V2.}
    \label{fig:variant_study_vidore_v2_jina}
    \vspace{-2mm}
\end{figure}

We conduct a variant analysis on the ViDoRe-V2 benchmark using the Jina-v4, as shown in Figure \ref{fig:variant_study_vidore_v2_jina}. 
We compare \method against several variants: \textbf{(I) \texttt{attention-ratio}}, a non-adaptive baseline that removes a fixed proportion of patches with the lowest attention scores; \textbf{(II) \texttt{attention-threshold}}, a static variant that discards patches based on a single, globally-defined attention threshold; \textbf{(III) \texttt{attention-threshold-nfp}}, which augments the static threshold approach by prepending a noise-filtering prompt to guide the model’s focus; and \textbf{(IV) \texttt{attention-threshold-adap}}, our adaptive pruning stage which is identical to DocPruner.

The complete \method framework consistently outperforms its standalone pruning component (\texttt{attention-threshold-adap}), underscoring the critical contribution of the subsequent merging stage, especially at high compression rates.
While our framework and its pruning-only counterpart perform identically at lower compression rates, a clear performance gap emerges as compression becomes more aggressive.
For example, at an approximately 80\% compression rate, \method maintains a robust nDCG@5 of 0.55, whereas the adaptive pruning-only variant drops to 0.51.
This confirms our central hypothesis: the initial pruning stage creates a high-quality, refined set of embeddings that allows the subsequent merging stage to operate more effectively, summarizing semantic concepts without the distortion caused by noise, a limitation inherent to pruning-only approaches at high compression.
See more discussion in Appendix \ref{app:variant_analysis}.

\subsubsection{Efficiency Analysis}
\label{sec:rq5_efficiency_analysis}
\method achieves a compelling balance between compression ratio and retrieval fidelity, effectively reducing storage costs by over half with minimal impact on accuracy. As detailed in Table \ref{tab:efficiency_analysis_vidore_v1_jina} (See details in Appendix \ref{app:efficiency_analysis}), the framework yields a substantial overall storage reduction of 54.60\% across the three base models, reaching a peak reduction of 58.88\% for ColQwen2.5. Despite this aggressive compression, retrieval performance remains remarkably robust, with the average nDCG@5 decreasing by only a marginal 0.45\% relative to the base models. While our approach increases the average per-document encoding latency from 0.46s to 0.69s, this overhead is incurred exclusively during the offline indexing stage and remains well within a perfectly acceptable range for real-world applications, especially considering that traditional OCR-based retrieval pipelines (\textit{e.g.,} OCR paired with BGE-M3 \cite{bge_m3}) can often exceed 7 seconds per page.

\begin{table}
 \centering
  \caption{
  Relative improvement of \method wrt. \textbf{performance}, \textbf{storage}, and \textbf{latency} to base models on ViDoRe-V1 ({We set adaptation factor as -0.75 and merging factor as 4; \color{RedOrange}orange} denotes better and {\color{green(pigment)}green} denotes worse).
  }
  \label{tab:efficiency_analysis_vidore_v1_jina}
  \vspace{-0.8em}
  \renewcommand\tabcolsep{5pt}
  \renewcommand\arraystretch{1.1}
  \footnotesize 
  \begin{tabular}{c|ccc} 
    \Xhline{1.2pt}
    \rowcolor{CadetBlue!20} 
    \textbf{$\Delta$} & \textbf{ColQwen} & \textbf{ColNomic} & \textbf{JinaV4}\\ 
    \Xhline{1pt}
    \textbf{nDCG@5} & {\color{green(pigment)} $\downarrow$0.27\%} & {\color{green(pigment)} $\downarrow$0.51\%} & {\color{green(pigment)} $\downarrow$0.57\%} \\
    \rowcolor{gray!10}\textbf{Storage} & {\color{RedOrange} $\downarrow$58.88\%} & {\color{RedOrange} $\downarrow$52.16\%} & {\color{RedOrange} $\downarrow$52.77\%}  \\
    \textbf{Latency} & {\color{green(pigment)} $\uparrow$56.10\%} & {\color{green(pigment)} $\uparrow$51.11\%} & {\color{green(pigment)} $\uparrow$47.06\%} \\
    \Xhline{1.2pt}
  \end{tabular}
  \vspace{-1.5em}
\end{table}

%% file: sections/5-conclusion.tex
\section{Conclusion}
\label{conclusion}
\vspace{-2mm}
In this work, we addressed the critical efficiency bottleneck in multi-vector VDR by proposing \method, a novel two-stage compression framework. 
Our framework uniquely synergizes the precision of pruning with the high-ratio compression of merging, following a ‘first refine, then compress’ paradigm. 
Through extensive experiments, we demonstrated that our approach not only extends the near-lossless compression range but also maintains superior performance at aggressive compression rates. 
Ultimately, \method{} provides a blueprint for advancing the practical applicability of multi-vector models.

%% file: sections/Appendix.tex
\newpage
\appendix
\hypersetup{linkcolor=black}
\etocdepthtag.toc{mtappendix}
\etocsettagdepth{mtchapter}{none}
\etocsettagdepth{mtappendix}{section}
\etocsettagdepth{mtappendix}{subsubsection}
\tableofcontents
\clearpage

{\large\textbf{Technical Appendices and Supplements}}

\section{Illustrative Examples}
\label{app:illustrative examples}

See Figures \ref{fig:example_vidorev1}, \ref{fig:example_vidorev2}, \ref{fig:example_jinavdr}, and \ref{fig:example_realmmrag}
for illustrative examples from representative VDR benchmarks. Figure~\ref{fig:example_prune1} illustrates the core rationale for our framework by comparing attention-based and random pruning. At low pruning rates (\textit{e.g.,} 10\%), attention-based pruning is clearly superior, as it selectively removes non-informative patches like whitespace. However, this advantage diminishes at higher ratios (\textit{e.g.,} 90\%), where both methods cause a catastrophic loss of content, as even critical information must be discarded to meet the aggressive compression target. This trend reveals a fundamental limitation of pruning-centric methods like DocPruner: while effective at filtering noise at moderate rates, they suffer a sharp performance cliff under high compression. Therefore, \method{} is specifically designed to circumvent this issue. We leverage pruning at a rate where it excels—refining the signal by removing noise—and then apply a merging stage to achieve higher compression by summarizing the remaining high-signal embeddings, thus avoiding the destructive information loss inherent in aggressive, pruning-only approaches.

\begin{figure*}[ht!]
  \centering
  \includegraphics[width=0.8\textwidth]{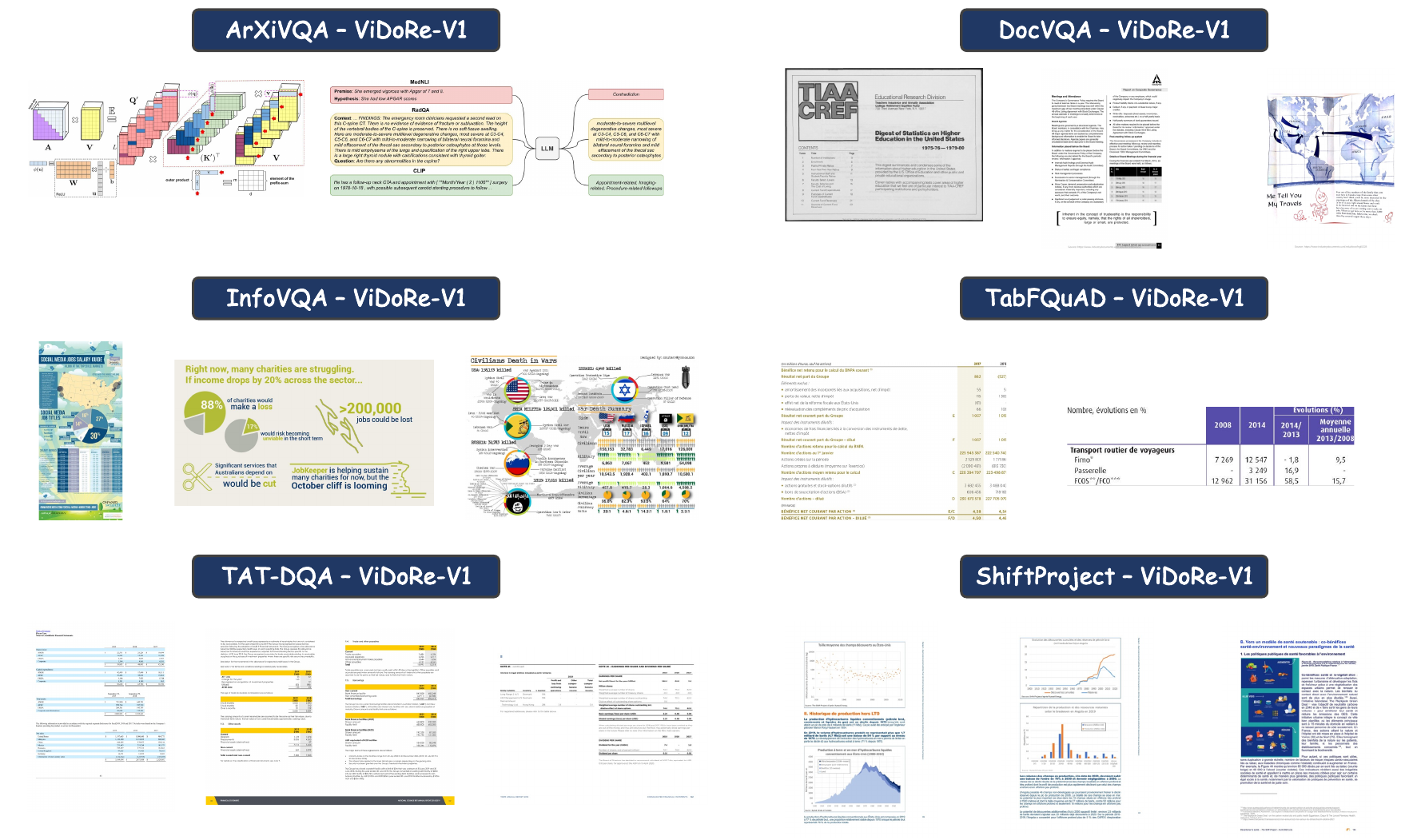}
  \caption{Illustrative examples of ViDoRe-V1 benchmark \cite{faysse2024colpali}.}
\label{fig:example_vidorev1}
\end{figure*}

\begin{figure*}[ht!]
  \centering
  \includegraphics[width=0.8\textwidth]{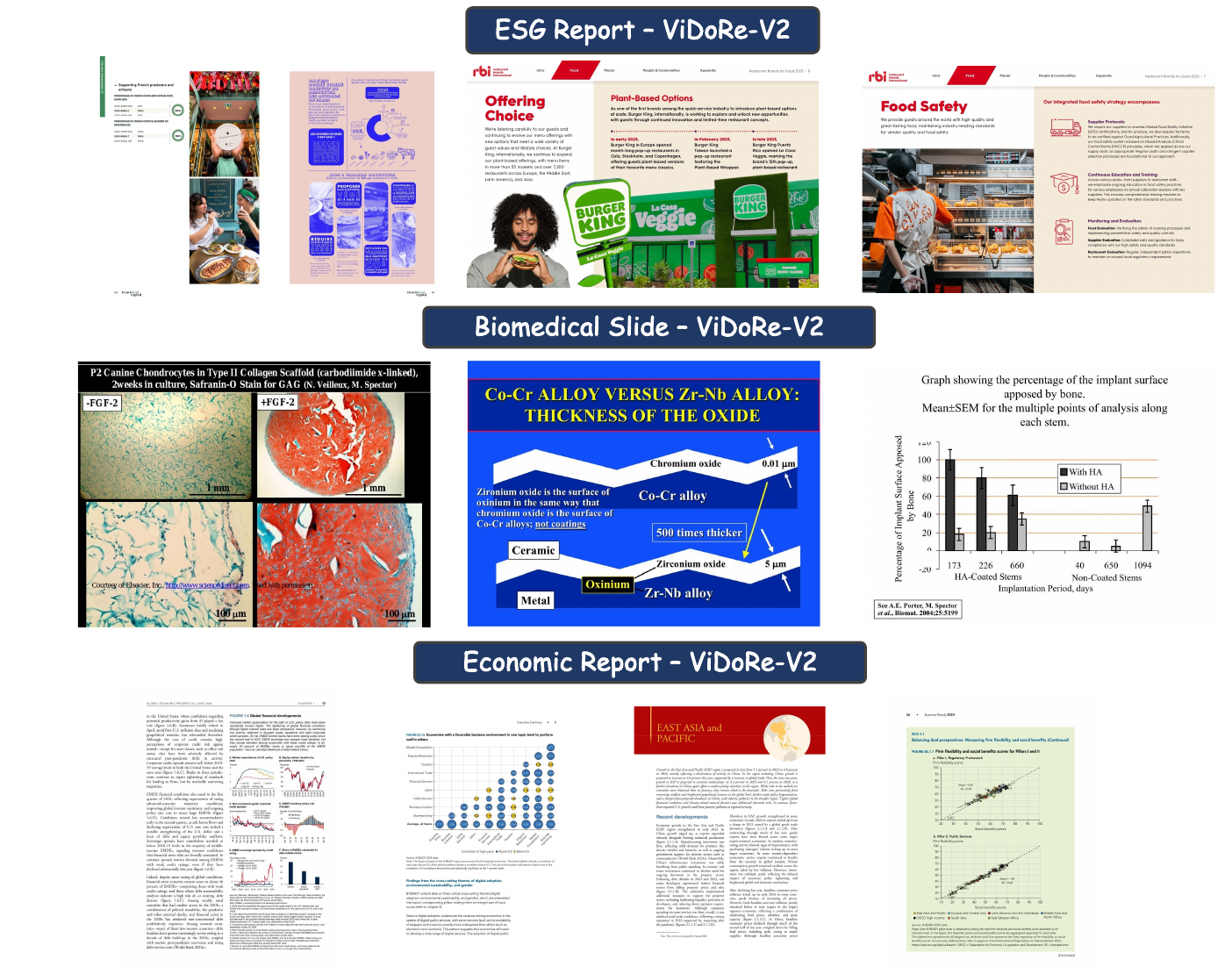}
  \caption{Illustrative examples of ViDoRe-V2 benchmark \cite{macé2025vidorebenchmarkv2raising}.}
\label{fig:example_vidorev2}
\vspace{-3mm}
\end{figure*}

\begin{figure*}[ht!]
  \centering
  \includegraphics[width=\textwidth]{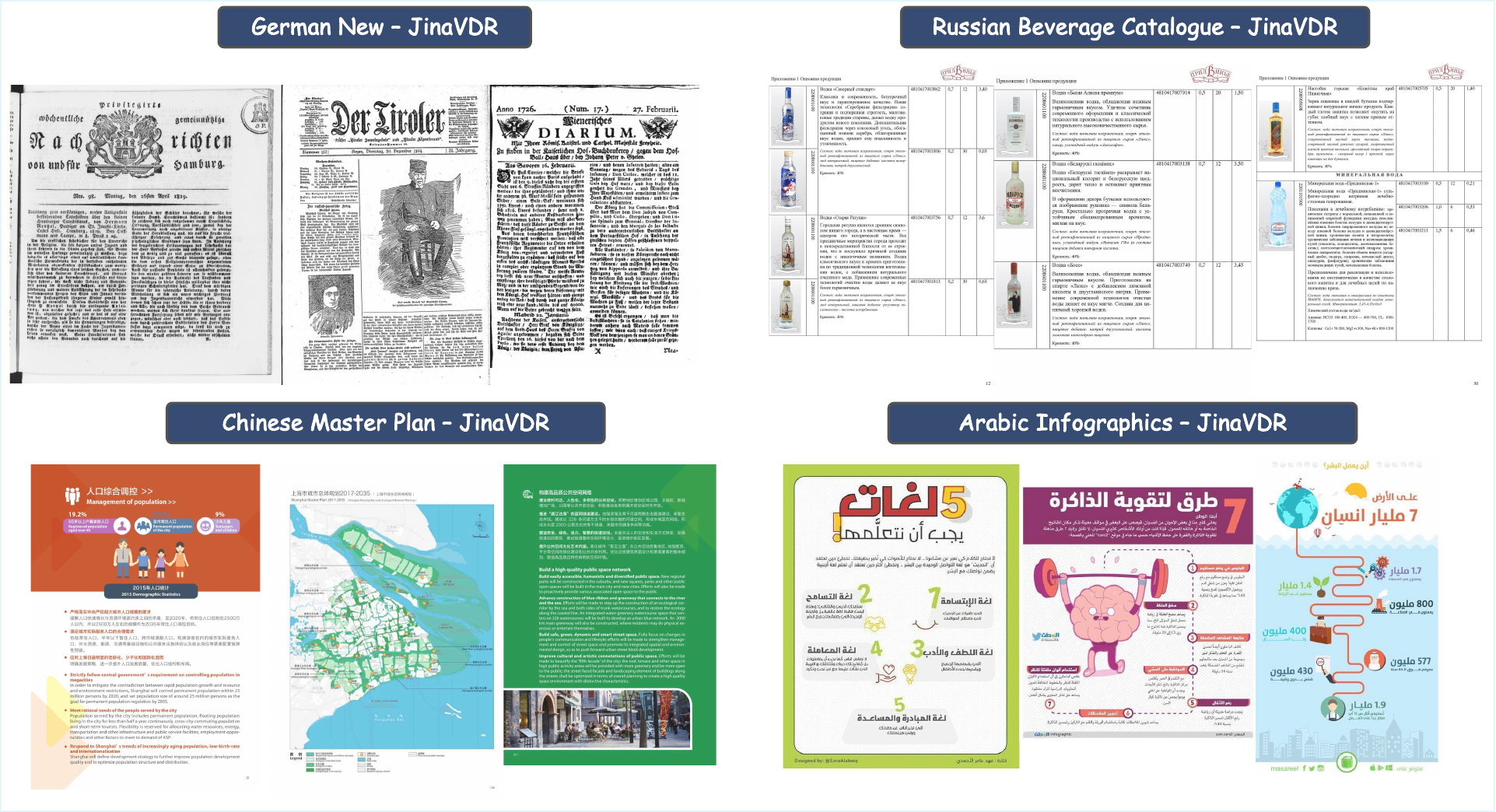}
  \caption{Illustrative examples of JinaVDR benchmark \cite{gunther2025jina}.}
\label{fig:example_jinavdr}
\end{figure*}

\begin{figure*}[ht!]
  \centering
  \includegraphics[width=\textwidth]{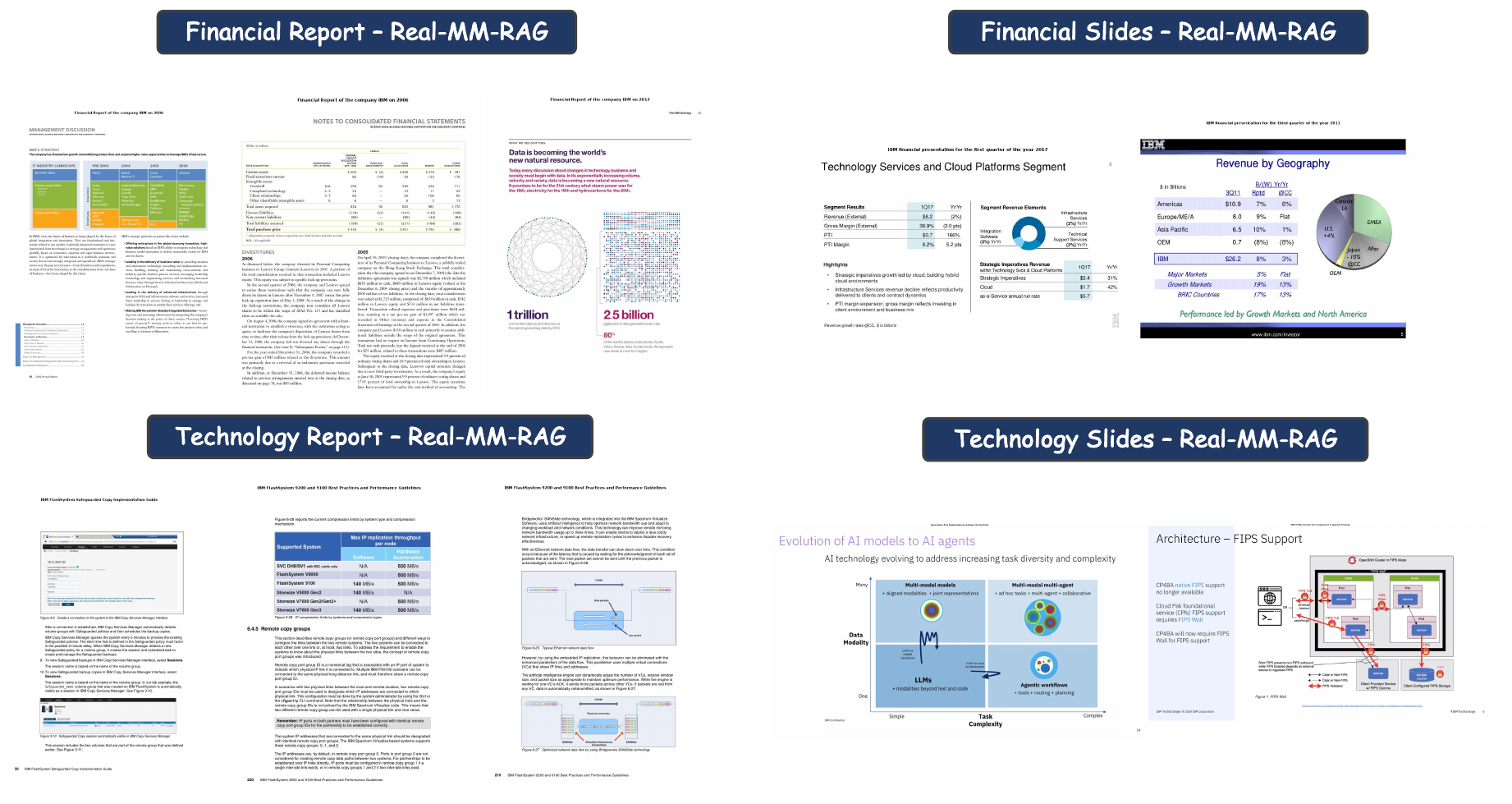}
  \caption{Illustrative examples of Real-MM-RAG benchmark \cite{wasserman2025realmmragrealworldmultimodalretrieval}.}
\label{fig:example_realmmrag}
\end{figure*}

\begin{figure*}[ht!]
  \centering
  \includegraphics[width=0.8\textwidth]{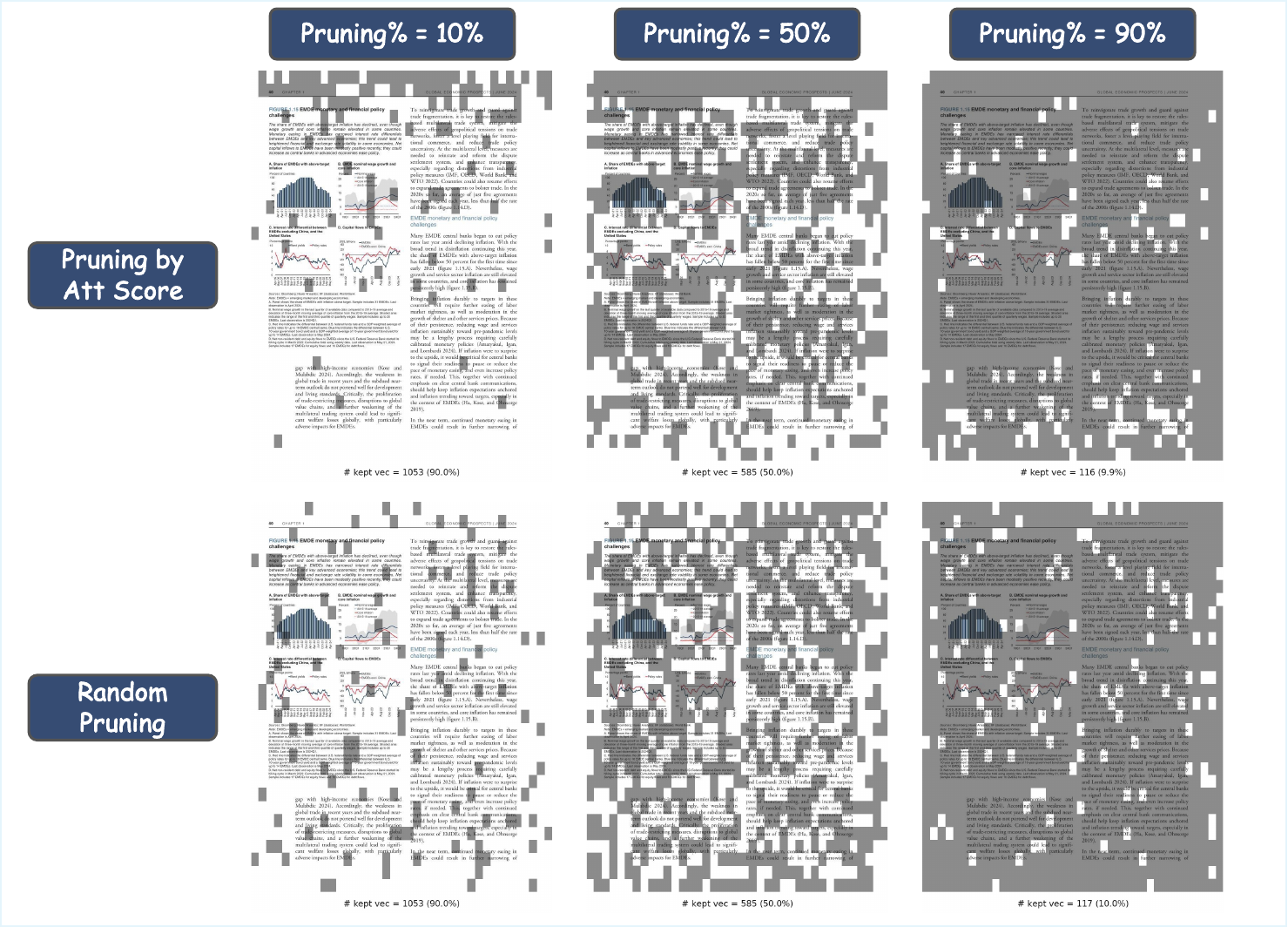}
  \caption{Illustrative example of pruned visual documents by different pruning ratios.}
\label{fig:example_prune1}
\end{figure*}

\section{More Related Work}
\label{app:more_related_work}

\subsection{Multi-Vector Retrieval}
\label{app:related_work_multi_vector_retrieval}

The concept of multi-vector retrieval, also known as late-interaction, was first popularized in the \textbf{text retrieval} domain by ColBERT \cite{khattab2020colbert}. 
This approach deviates from single-vector methods by representing each document as a "bag" of contextualized token embeddings. 
Instead of pre-aggregating information into a single dense vector, it computes relevance scores at a fine-grained level using the \texttt{MaxSim} operator (as shown in Eq. \ref{eq:maxsim}), enabling more nuanced term-level matching \cite{clavie2024reducing,clavie2025simple,jha2024jina,bge_m3,xie2025investigating,li2023slim}. 
However, this paradigm introduces prohibitive storage and computational overhead, which has spurred a significant stream of research focused on optimization. These efforts can be broadly categorized into three directions.
The first focuses on \textit{indexing and search acceleration}. 
ColBERTv2 \cite{santhanam2021colbertv2} introduced centroid-based compression, where each embedding $\mathbf{d}_j$ is approximated by its nearest centroid $\mathbf{c}_k$ and a compressed residual $\mathbf{r}_j$, i.e., $\mathbf{d}_j \approx \mathbf{c}_k + \mathbf{r}_j$. 
Building on this, PLAID \cite{santhanam2022plaid} accelerated retrieval by using these centroids for efficient pruning of irrelevant documents.
Other works have explored reducing the interaction space itself. For instance, CITADEL \cite{li2022citadel} introduced dynamic lexical routing to constrain token interactions to only those sharing a predicted "lexical key," while MUVERA \cite{jayaram2024muvera} proposed converting multi-vector sets into single Fixed Dimensional Encodings (FDEs) whose inner product approximates the original multi-vector similarity, thereby enabling the use of standard MIPS solvers.
A second line of work aims to \textit{enhance the scoring function's expressiveness}. For example, TRIAL \cite{kang2025trial} explicitly models token relations and incorporates learned importance weights for each query token to improve accuracy.
Finally, another major direction, which our work builds upon, is \textit{reducing the number of stored vectors per document via offline token pruning or merging strategies} \cite{clavie2024reducing,veneroso2025crisp,macavaney2025efficient}, directly tackling the storage bottleneck.

This paradigm was naturally extended to the \textbf{multimodal domain} to handle the complexities of visually-rich documents, where fine-grained matching is critical \cite{kim2025hybrid,kolavi2025m3dr,liu2025rematch,clavie2025lir}. 
ColPali \cite{faysse2024colpali} pioneered this direction by adapting the ColBERT framework to VDR, using a VLM to generate multi-vector patch embeddings directly from document images. 
This allows the system to match query token embeddings against specific image patch embeddings, preserving spatial and semantic details lost in page-level models. 
This breakthrough spurred a new wave of VDR models. 
For example, recent works like ColNomic \cite{nomicembedmultimodal2025} have enhanced performance by training on vast corpora of interleaved text and image data, while Jina-v4 \cite{gunther2025jina} proposed a unified architecture capable of producing both multi-vector and single-vector outputs for greater flexibility.
Besides, Llama Nemoretriever Colembed \cite{xu2025llama} advanced further by replacing the standard causal attention in VLMs with bidirectional attention to better suit retrieval tasks. 
Similarly, ModernVBERT \cite{teiletche2025modernvbert} demonstrated that a smaller, purpose-built bidirectional encoder can achieve competitive performance, highlighting the synergy between encoder architectures and late interaction.
Other works have focused on refining the training process and interaction mechanism, such as ColMate \cite{masry2025colmate}, which introduced a masked OCR pretraining objective and a \texttt{TopKSim} operator to better handle noisy patch-based tokenization.
From an efficiency standpoint, ColFlor \cite{masry2024colflor} developed a significantly smaller 174M parameter model that maintains competitive performance on text-rich documents. More recently, the multi-vector paradigm has been extended to omni-modal settings, with works like CLaMR \cite{wan2025clamr} and Omni-Embed-Nemotron \cite{xu2025omni} developing unified models for joint retrieval across text, image, audio, and video, signaling a promising direction for future research.


\subsection{Multimodal Retrieval Models}
\label{app:related_work_multimodal_retrieval_models}
Universal multimodal retrieval models aim to create a unified embedding space for various data types, most commonly text, images, and videos, in addition to the visually-rich documents central to this paper \cite{zhao2023retrieving,mei2025survey}. 
Each modality presents distinct challenges. 
For instance, image-text retrieval often focuses on matching textual descriptions to static visual scenes, capturing objects, attributes, and their relationships \cite{cao2022image,ying2022survey}. 
Video retrieval introduces a temporal dimension, requiring the model to understand actions, events, and narrative progression over time \cite{zhu2023deep}. 
In contrast, VDR is particularly complex as it must jointly comprehend dense, structured textual information and complex spatial layouts, such as tables, figures, and forms, which are often lost in standard image or text encoders  \cite{faysse2024colpali,chen2025cmrag,sun2025unveil}.

The foundation of modern multimodal retrieval was laid by dual-encoder models like CLIP \cite{radford2021learning}, which pioneered learning aligned representations through large-scale contrastive training on image-text pairs. 
Building on this, recent works have increasingly leveraged the power of LVLMs to create more universal embeddings \cite{zhu2025freeret,liu2025any,ma2024unifying,xiao2025scaling}. 
For instance, E5-V \cite{jiang2024e5} employs unimodal contrastive learning on the language component of a LVLM to bridge the modality gap. 
VLM2Vec series \cite{jiang2024vlm2vec,meng2025vlm2vec} further explores this direction by introducing a comprehensive framework and benchmark (MMEB) to repurpose pre-trained VLMs into powerful embedding models. 
Other methods focus on refining the training dynamics; for example, UniME-V1 \cite{gu2025breaking} and QQMM \cite{xue2025improve} introduce advanced hard negative mining strategies and gradient amplification techniques to enhance the discriminative power of the learned embeddings. 
Furthermore, UniME-V2 \cite{gu2025unime} leverages an MLLM-as-a-Judge for sophisticated hard negative mining, MM-Embed \cite{lin2024mm} introduces modality-aware hard negative mining to mitigate modality bias, RzenEmbed \cite{jian2025rzenembed} introduces a hardness-weighted contrastive loss, B3 \cite{thirukovalluru2025breaking} proposes smart batch mining to construct more effective training batches, and LamRA \cite{liu2025lamra} adopts a progressive two-stage training strategy.
Some efforts are data-centric; for example, MegaPairs \cite{zhou2025megapairs} and GME \cite{zhang2024gme} introduce pipelines to synthesize massive, high-quality fused-modal training data and optimize the training data composition to enhance model robustness.
Concurrently, other approaches like ReMatch \cite{liu2025rematch}, CAFe \cite{yu2025cafe}, and VladVA \cite{ouali2025vladva} unify retrieval with generative objectives by using the MLLM to predict relevance or by employing a joint contrastive-autoregressive loss.
More recently, a new trend leverages the explicit reasoning capabilities of MLLMs, with works like Think-Then-Embed \cite{cui2025tte}, UME-R1 \cite{lan2025ume}, and Reasoning Guided Embeddings (RGE) \cite{liu2025reasoning} incorporating intermediate rationale generation to produce more informative representations.

\subsection{Large Vision-Language Models}
\label{app:related_work_lvlm}

In recent years, Large Vision-Language Models (LVLMs) have demonstrated rapid advancements in general-purpose capabilities, significantly benefiting a wide array of downstream domains such as coding \cite{lin2025autop2c,wang2025code}, scientific reasoning \cite{yan2024survey,yan2024errorradar,yan2025position,yang2025r1}, education \cite{dai2025physicsarena,su2025essayjudge,su2025cafes}, urban computing \cite{zhu2025cityverse,li2025urban,yan2024urbanclip,yan2024georeasoner}, trustworthy system \cite{huo2025mmunlearner,huo2024mmneuron,huang2025pierce,liu2025vlamark}, and even agentic system \cite{liu2026lmagent,long2025seeing,yan2025mathagent}. 
This progress is exemplified by a series of state-of-the-art open-source models, each introducing unique technical innovations. 
For instance, InternVL3~\cite{zhu2025internvl3} pioneers a native multimodal pre-training paradigm that jointly learns visual and linguistic capabilities in a single stage, fundamentally improving cross-modal alignment. 
To enhance efficiency and specialization, DeepSeek-VL2~\cite{wu2024deepseekvl2} integrates a Mixture-of-Experts (MoE) architecture with a dynamic tiling strategy for high-resolution vision. 
Architectural flexibility is further explored by Qwen2-VL~\cite{wang2024qwenvl2}, which introduces a Naive Dynamic Resolution mechanism and Multimodal Rotary Position Embedding (M-RoPE) to handle varied image sizes, and Gemma 3~\cite{Kamath2025Gemma3}, which achieves efficient long-context understanding by strategically increasing the ratio of local to global attention layers. 
Building on these foundations, Qwen2.5-VL~\cite{bai2025qwenvl25} focuses on fine-grained perception through dynamic resolution and absolute time encoding for precise long-video analysis.
To boost reasoning, InternVL3.5~\cite{wang2025internvl35} proposes a Cascade Reinforcement Learning (Cascade RL) framework combining offline and online stages for more stable and refined alignment.
Complementing these advancements, Kimi K2.5~\cite{team2026kimi25} introduces the Agent Swarm, a parallel agent orchestration framework that decomposes complex tasks for concurrent execution.
The sophisticated features developed in these models, such as enhanced fine-grained perception, superior cross-modal alignment, efficient long-context processing, and advanced reasoning, are poised to significantly benefit the field of multimodal retrieval. 
By leveraging these capabilities, retrieval systems can achieve more robust and accurate results when searching and ranking complex image-text data.

\section{Algorithm Workflow}
\label{app:algo_workflow}
We formalize the complete workflow of our proposed \method framework in two distinct algorithms. 
\textbf{Algorithm~\ref{alg:prune_then_merge_compression}} details the offline compression phase, where \method{} generates a highly compact set of document embeddings through its sequential two-stage process.
Subsequently, \textbf{Algorithm~\ref{alg:prune_then_merge_scoring}} illustrates the online retrieval phase, where the final relevance score is efficiently computed via a \texttt{MaxSim} operation using this compressed set of embeddings.

\begin{algorithm*}[!h]
\caption{The \method{} Compression (Offline Indexing Phase)}
\label{alg:prune_then_merge_compression}
\DontPrintSemicolon
\SetAlgoLined

\KwIn{A document page $d$;\\
      A VLM encoder $\Phi(\cdot)$ outputting patch embeddings and attention weights;\\
      A pruning sensitivity controller hyperparameter $k$; \\
      A merging factor hyperparameter $m$.}
\KwOut{A compressed set of patch embeddings $\mathbf{D}''$.}

\BlankLine

\tcc{\textcolor{blue}{--- STAGE 1: ADAPTIVE PRUNING ---}}
\tcc{\textcolor{purple}{Step 1.1: VLM Forward Pass to get Embeddings and Attention}}
$\{\mathbf{D}, \mathbf{A}\} \leftarrow \Phi(d)$ \tcp*{Extract initial embeddings $\mathbf{D} = \{\mathbf{d}_j\}_{j=1}^{N_p}$ and attention $\mathbf{A}$}

\BlankLine
\tcc{\textcolor{purple}{Step 1.2: Quantify Patch Importance via Global Token Attention}}
Let $g$ be the index of the global token (e.g., \texttt{[EOS]})\;
Initialize an empty list of importance scores $\mathcal{I}_d$\;
\For{$j \leftarrow 1$ \KwTo $N_p$}{
    $\bar{\mathbf{A}}_{g, j} \leftarrow \frac{1}{H} \sum_{h=1}^{H} \mathbf{A}_{h, g, j}$ \tcp*{Average attention over $H$ heads}
    $I(\mathbf{d}_j) \leftarrow \bar{\mathbf{A}}_{g, j}$\;
    Append $I(\mathbf{d}_j)$ to $\mathcal{I}_d$\;
}

\BlankLine
\tcc{\textcolor{purple}{Step 1.3: Compute Adaptive Threshold and Prune}}
$\mu_d \leftarrow \text{mean}(\mathcal{I}_d)$; \quad $\sigma_d \leftarrow \text{std\_dev}(\mathcal{I}_d)$\;
$\tau_d \leftarrow \mu_d + k \cdot \sigma_d$ \tcp*{Document-specific pruning threshold}
$\hat{\mathbf{D}}' \leftarrow \{\mathbf{d}_j \in \mathbf{D} \mid I(\mathbf{d}_j) > \tau_d\}$ \tcp*{Preliminary pruned set}

\BlankLine
\tcc{\textcolor{purple}{Step 1.4: Finalize Pruned Set with Robustness Guarantee}}
\eIf{$\hat{\mathbf{D}}' = \emptyset$}{
    $j^* \leftarrow \underset{j}{\arg\max} \, I(\mathbf{d}_j)$\;
    $\mathbf{D}' \leftarrow \{\mathbf{d}_{j^*}\}$ \tcp*{Keep single most important patch}
}{
    $\mathbf{D}' \leftarrow \hat{\mathbf{D}}'$
}
$N_p' \leftarrow |\mathbf{D}'|$ \tcp*{Number of patches after pruning}

\BlankLine

\tcc{\textcolor{blue}{--- STAGE 2: HIERARCHICAL MERGING ---}}
\tcc{\textcolor{purple}{Step 2.1: Check if Merging is Applicable}}
\uIf{$N_p' < m$ \textbf{or} $m \le 1$}{
    $\mathbf{D}'' \leftarrow \mathbf{D}'$ \tcp*{Skip merging if set is already small}
}
\Else{
    \tcc{\textcolor{purple}{Step 2.2: Determine Target Number of Clusters}}
    $N_p'' \leftarrow \max(1, \lfloor N_p' / m \rfloor)$\;
    
    \BlankLine
    \tcc{\textcolor{purple}{Step 2.3: Perform Hierarchical Clustering}}
    \tcp{Normalize embeddings, compute distance matrix, apply linkage}
    $\text{labels} \leftarrow \text{HierarchicalClustering}(\mathbf{D}', \text{num\_clusters}=N_p'')$\;
    
    \BlankLine
    \tcc{\textcolor{purple}{Step 2.4: Compute Cluster Centroids}}
    $\mathbf{D}'' \leftarrow \{\}$\;
    \For{$c \leftarrow 1$ \KwTo $N_p''$}{
        $C_c \leftarrow \{\mathbf{d}_j \in \mathbf{D}' \mid \text{label}(\mathbf{d}_j) = c\}$\;
        $\mathbf{d}''_c \leftarrow \frac{1}{|C_c|} \sum_{\mathbf{d}_j \in C_c} \mathbf{d}_j$ \tcp*{Compute centroid}
        $\mathbf{D}'' \leftarrow \mathbf{D}'' \cup \{\mathbf{d}''_c\}$\;
    }
}
\Return $\mathbf{D}''$\;
\end{algorithm*}

\begin{algorithm*}[!h]
\caption{Scoring with Pruned-and-Merged Embeddings (Online Retrieval Phase)}
\label{alg:prune_then_merge_scoring}
\DontPrintSemicolon
\SetAlgoLined

\KwIn{A textual query $q$;\\
      The compressed document embedding set $\mathbf{D}''$ (from Alg. \ref{alg:prune_then_merge_compression});\\
      A VLM encoder $\Phi(\cdot)$ for query encoding.}
\KwOut{The final relevance score $s''(q, d)$.}

\BlankLine

\tcc{\textcolor{blue}{Step 1: Encode Query}}
$\mathbf{Q} \leftarrow \Phi(q)$ \tcp*{Encode $q$ into token embeddings $\mathbf{Q} = \{\mathbf{q}_i\}$}

\BlankLine

\tcc{\textcolor{blue}{Step 2: Compute Relevance Score with Compressed Embeddings}}
$s''(q, d) \leftarrow 0$\;
\For{$\mathbf{q}_i \in \mathbf{Q}$}{
    max\_sim $\leftarrow -\infty$\;
    \For{$\mathbf{d}''_c \in \mathbf{D}''$}{
        sim $\leftarrow \mathbf{q}_i^\top \mathbf{d}''_c$\;
        \If{sim $>$ max\_sim}{
            max\_sim $\leftarrow$ sim\;
        }
    }
    $s''(q, d) \leftarrow s''(q, d) +$ max\_sim \tcp*{Aggregate max similarity per query token}
}

\BlankLine
\Return $s''(q, d)$\;
\end{algorithm*}

\clearpage
\section{More Theoretical Analysis}
\label{app:more_theory}

This section provides a more detailed theoretical underpinning for the \method{} framework, expanding upon the analysis presented in Section \ref{sec:theoretical_guarantee}. We begin by formalizing the core principles of information theory that motivate our approach, then present theorems and corollaries that justify the efficacy of our two-stage decomposition.

\subsection{Theoretical Foundations: IB and RD Theory}

\paragraph{Information Bottleneck (IB) Principle.}
The IB principle~\cite{tishby2000information,tishby2015deep,gordon2003applying} formalizes the trade-off between compression and prediction. Given a source random variable $\mathbf{X}$ and a target variable $\mathbf{Y}$, the goal is to find a compressed representation $\mathbf{Z}$ of $\mathbf{X}$ that retains the maximum possible information about $\mathbf{Y}$. This is expressed as minimizing the Lagrangian:
\begin{equation}
    \mathcal{L}_{\text{IB}}(\mathbf{Z}) = I(\mathbf{Z}; \mathbf{X}) - \beta I(\mathbf{Z}; \mathbf{Y})
\end{equation}
where $I(\cdot;\cdot)$ is the mutual information and $\beta > 0$ is a Lagrange multiplier. In our context, $\mathbf{X}$ is the full patch set $\mathbf{D}$, $\mathbf{Z}$ is the compressed set $\mathbf{D}''$, and $\mathbf{Y}$ is the relevance variable, dependent on an unknown query $q$.

\paragraph{Rate-Distortion (RD) Theory.}
RD theory~\cite{ortega1998rate,yeung2008rate} addresses lossy data compression. Given a source $\mathbf{X}$ and a distortion measure $d(\mathbf{x}, \hat{\mathbf{x}})$, the rate-distortion function $R(\Delta)$ is the minimum rate required to represent the source such that the expected distortion does not exceed $\Delta$:
\begin{equation}
    R(\Delta) = \min_{p(\hat{\mathbf{x}}|\mathbf{x}) : \mathbb{E}[d(\mathbf{X}, \hat{\mathbf{X}})] \le \Delta} I(\mathbf{X}; \hat{\mathbf{X}})
\end{equation}
Vector Quantization (VQ) is a practical approach to solving the RD problem for continuous variables, where the goal is to find a finite "codebook" (our $\mathbf{D}''$) that best represents the original data distribution (our $\mathbf{D}'$).

\subsection{Theoretical Decomposition of the Framework}
Our framework's strength lies in decomposing the complex IB problem into two sequential, manageable sub-problems, justified by the following theorems.

\begin{theorem}
\label{thm:pruning}
\textbf{(Information-Preserving Noise Filtering)}
Let the full patch set $\mathbf{D}$ be a disjoint union of a signal set $\mathbf{D}_{\text{sig}}$ and a noise set $\mathbf{D}_{\text{noi}}$. Let the importance score $I(\mathbf{d}_j)$ be a proxy for the information a patch $\mathbf{d}_j$ provides about the document's global semantics $\mathbf{h}_{\emph{eos}}$. If the noise set satisfies $\forall \mathbf{d}_j \in \mathbf{D}_{\emph{noi}}$, $I(\mathbf{d}_j; \mathbf{h}_{\emph{eos}}) < \epsilon$ for some small $\epsilon > 0$, then the pruned set $\mathbf{D}'$ generated by our adaptive thresholding mechanism retains almost all the information of the original set $\mathbf{D}$ with respect to $\mathbf{h}_{\emph{eos}}$. Formally:
\begin{equation}
\begin{split}
    I(\mathbf{D}; \mathbf{h}_{\emph{eos}}) - I(\mathbf{D}'; \mathbf{h}_{\emph{eos}}) & \le \sum_{\mathbf{d}_j \in \mathbf{D} \setminus \mathbf{D}'} I(\mathbf{d}_j; \mathbf{h}_{\emph{eos}} | \mathbf{D}') \\
    & \approx 0
\end{split}
\end{equation}
\end{theorem}
\begin{proof}
By the chain rule for mutual information, $I(\mathbf{D}; \mathbf{h}_{\text{eos}}) = I(\mathbf{D}'; \mathbf{h}_{\text{eos}}) + I(\mathbf{D} \setminus \mathbf{D}'; \mathbf{h}_{\text{eos}} | \mathbf{D}')$. Our pruning stage $g_p$ is designed to discard patches $\mathbf{d}_j$ where their proxy score is low, implying $I(\mathbf{d}_j; \mathbf{h}_{\text{eos}})$ is small. Since these patches are conditionally independent of $\mathbf{h}_{\text{eos}}$ given the high-information set $\mathbf{D}'$, their conditional mutual information term $I(\mathbf{d}_j; \mathbf{h}_{\text{eos}} | \mathbf{D}')$ is also negligible. Summing these negligible terms leads to the result.
\end{proof}

\begin{theorem}
\label{thm:merging}
\textbf{(Optimal Redundancy Reduction via Quantization)}
For a given set of $N_p'$ vectors $\mathbf{D}'$ and a target codebook size of $N_p''$, the set of centroids $\mathbf{D}''$ obtained by partitioning $\mathbf{D}'$ into $N_p''$ clusters $\{C_c\}$ and computing the mean for each cluster is the optimal solution to the vector quantization problem that minimizes the Mean Squared Error (MSE) distortion.
\begin{equation}
\begin{split}
    \mathbf{D}'' = \underset{\hat{\mathbf{D}}'' : |\hat{\mathbf{D}}''|=N_p''}{\arg\min} & \sum_{c=1}^{N_p''} \sum_{\mathbf{d}_j \in C_c} ||\mathbf{d}_j - \hat{\mathbf{d}}''_c||^2_2
\end{split}
\end{equation}
\end{theorem}
\begin{proof}
This is a standard result from vector quantization theory. For any given cluster $C_c$, the partial derivative of the inner sum with respect to the centroid $\mathbf{d}''_c$ is $2 \sum (\mathbf{d}_j - \mathbf{d}''_c)$. Setting this to zero to find the minimum yields $\mathbf{d}''_c = \frac{1}{|C_c|} \sum_{\mathbf{d}_j \in C_c} \mathbf{d}_j$. Our hierarchical merging stage directly implements this optimal solution.
\end{proof}

\subsection{Analysis of Synergistic Gain}
The sequential application of these two stages yields a synergistic gain that neither stage could achieve alone. This can be formalized by analyzing the distortion with respect to the \textit{true signal} $\mathbf{D}_{\text{sig}}$.

\begin{corollary}[Synergistic Distortion Reduction]
\label{cor:synergy}
Let $\mathbf{D}''_{\emph{ours}} = g_m(g_p(\mathbf{D}))$ be the output of our two-stage framework, and $\mathbf{D}''_{\emph{naive}} = g_m(\mathbf{D})$ be the output of a naive single-stage merge. The expected distortion of our representation with respect to the true signal set $\mathbf{D}_{\emph{sig}}$ is strictly lower than that of the naive approach.
\begin{equation}
\begin{split}
    \mathbb{E}_{\mathbf{d}_s \in \mathbf{D}_{\emph{sig}}} [ d(\mathbf{d}_s, \mathbf{D}''_{\emph{ours}}) ] & < \\ 
    \mathbb{E}_{\mathbf{d}_s \in \mathbf{D}_{\emph{sig}}} [ & d(\mathbf{d}_s, \mathbf{D}''_{\emph{naive}}) ]
\end{split}
\end{equation}
where $d(\mathbf{x}, \mathcal{Y}) = \min_{\mathbf{y} \in \mathcal{Y}} ||\mathbf{x} - \mathbf{y}||^2_2$ is the quantization error.
\end{corollary}
\begin{proof}
A centroid of a naive merge, $\mathbf{d}''_c(\mathbf{D})$, is a biased estimator of the true signal's center of mass.
\begin{equation}
\begin{split}
    \mathbf{d}''_{c, \text{naive}} &= \frac{1}{|C_c|} \sum_{\mathbf{d}_j \in C_c} \mathbf{d}_j \\
    &= \mathbf{d}''_{c, \text{sig}} + \underbrace{\frac{1}{|C_c|} \sum_{\mathbf{d}_k \in C_c \cap \mathbf{D}_{\text{noi}}} (\mathbf{d}_k - \mathbf{d}''_{c, \text{sig}})}_\text{Bias Term}
\end{split}
\end{equation}
In our framework, the pruning stage first ensures $\mathbf{D}' \approx \mathbf{D}_{\text{sig}}$. Therefore, the centroids $\mathbf{d}''_c(\mathbf{D}')$ are computed on a nearly noise-free set, making them approximately unbiased estimators. Since our centroids are closer to the true signal distribution, the expected distortion for any signal vector $\mathbf{d}_s$ will be lower.
\end{proof}

Finally, the Data Processing Inequality dictates that for any Markov chain $\mathbf{X} \to \mathbf{Y} \to \mathbf{Z}$, we have $I(\mathbf{X}; \mathbf{Z}) \le I(\mathbf{X}; \mathbf{Y})$. In our framework, the compression forms the Markov chain $\mathbf{D} \to \mathbf{D}' \to \mathbf{D}''$. This implies:
\begin{equation}
    I(\mathbf{D}; \mathbf{h}_{\text{eos}}) \ge I(\mathbf{D}'; \mathbf{h}_{\text{eos}}) \ge I(\mathbf{D}''; \mathbf{h}_{\text{eos}})
\end{equation}
The goal of an effective compression scheme is to make these inequalities as close to equalities as possible. Our \method{} framework achieves this by ensuring the first inequality is nearly an equality (via Theorem~\ref{thm:pruning}) and the second inequality represents a highly efficient rate-distortion trade-off (via Theorem~\ref{thm:merging} on a clean signal).

\clearpage
\section{Benchmark Details}
\label{app:benchmark_details}
This section provides detailed descriptions of the benchmarks evaluated in this paper.

\subsection{ViDoRe-V1 Benchmark}
\label{app:vidore_v1_benchmark}

ViDoRe-V1, the first Visual Document Retrieval Benchmark, is introduced to address a critical gap in existing evaluation suites. While many benchmarks focus on text embedding models, ViDoRe-V1 evaluates the entire document ingestion and retrieval pipeline, emphasizing the ability of a system to understand both textual content and vital visual cues like figures, tables, and layout. It is designed to assess retrievers on their capacity to process visually rich information as a human would, featuring a diverse collection of 10 sub-datasets that span various topics, modalities, and languages.

The benchmark includes three datasets adapted from established Visual Question Answering (VQA) tasks, each subsampled for consistency. \textbf{ArXivQA}\footnote{\url{https://huggingface.co/datasets/vidore/arxivqa_test_subsampled}} is a VQA dataset built from figures in arXiv publications, with questions synthetically generated by GPT-4 Vision \cite{li2024multimodalarxiv}. \textbf{DocVQA}\footnote{\url{https://huggingface.co/datasets/vidore/docvqa_test_subsampled}} is sourced from the DocVQA test set, using images from the UCSF Industry Documents Library with manually annotated questions and answers \cite{mathew_docvqa_2020}. Similarly, \textbf{InfoVQA}\footnote{\url{https://huggingface.co/datasets/vidore/infovqa_test_subsampled}} is derived from the InfoVQA test set and comprises infographics collected from the internet, also featuring manual annotations \cite{mathew_infographicvqa_2021}. Each of these test sets was subsampled to 500 query-document pairs to ensure homogeneity.

Two datasets specifically target retrieval from complex tabular and textual content. \textbf{TabFQuAD} is designed to assess table QA models in French, reflecting realistic industry scenarios. It combines human-annotated queries with additional questions generated by GPT-4V, with a test set of 280 pairs. In contrast, \textbf{TAT-DQA}\footnote{\url{https://huggingface.co/datasets/vidore/tatdqa_test}} is a large-scale dataset from real-world financial reports, focusing on rich tabular content and numerical reasoning, with questions annotated by finance experts \cite{tatqa,zhu2022towards}. Its full test set of 1,663 pairs is retained, as its complexity and domain-specificity closely align with practical, high-stakes retrieval use cases.

To evaluate multilingual performance, ViDoRe-V1 incorporates \textbf{ShiftProject}\footnote{\url{https://huggingface.co/datasets/vidore/shiftproject_test}}, a topic-specific benchmark in French. This dataset contains five large reports (totaling approximately 1,000 pages) from the Shift Project concerning environmental topics. It includes 100 question-answer pairs generated by the Claude-3 Sonnet model, which were then extensively filtered by human annotators to ensure high quality and relevance.

Finally, a suite of four \textbf{SyntheticDocQA} datasets simulates retrieval tasks in realistic industrial applications across diverse domains. Each dataset was constructed by crawling 1,000 PDFs from the internet using a topic-specific query, randomly sampling 1,000 pages, and then generating 100 high-quality question-answer pairs using Claude-3 Sonnet. The topics are designed to benchmark performance on specific document types: \textbf{Artificial Intelligence}\footnote{\url{https://huggingface.co/datasets/vidore/syntheticDocQA_artificial_intelligence_test}}, \textbf{Energy}\footnote{\url{https://huggingface.co/datasets/vidore/syntheticDocQA_energy_test}} (technical documentation), \textbf{Government Reports}\footnote{\url{https://huggingface.co/datasets/vidore/syntheticDocQA_government_reports_test}} (administrative and legal documents), and \textbf{Healthcare Industry}\footnote{\url{https://huggingface.co/datasets/vidore/syntheticDocQA_healthcare_industry_test}} (medical documents).

\subsection{ViDoRe-V2 Benchmark}
\label{app:vidore_v2_benchmark}

ViDoRe-V2 was developed to address the limitations of existing benchmarks and better reflect real-world retrieval challenges. It moves beyond common pitfalls such as the over-reliance on extractive queries, a bias towards single-page contexts, and the quality issues inherent in purely synthetic query generation. To create a more robust and realistic evaluation, ViDoRe-V2 introduces several key innovations. Queries are generated with limited document context (\textit{e.g.,} summaries or metadata) to mimic users who are unfamiliar with the corpus, thus reducing extractive bias. It also emphasizes long-form and cross-document queries. Critically, it employs a hybrid synthetic and human-in-the-loop methodology, where synthetically generated queries undergo extensive human review to ensure high quality and relevance.

The benchmark includes a set of three multilingual, topic-specific datasets created through this semi-synthetic process. Each dataset focuses on a distinct domain and contains queries in English, French, German, and Spanish, which were generated by translating original queries with GPT-4o. The datasets are: \textbf{esg\_reports\_v2}\footnote{\url{https://huggingface.co/datasets/vidore/esg_reports_v2}}, which focuses on ESG reports from the fast-food industry; \textbf{biomedical\_lectures\_v2}\footnote{\url{https://huggingface.co/datasets/vidore/biomedical_lectures_v2}}, centering on MIT biomedical lectures about tissue interactions; and \textbf{economics\_reports\_v2}\footnote{\url{https://huggingface.co/datasets/vidore/economics_reports_v2}}, which covers world economic reports from 2024.

To provide a non-synthetic baseline, ViDoRe-V2 features \textbf{esg\_reports\_human\_labeled\_v2}\footnote{\url{https://huggingface.co/datasets/vidore/esg_reports_human_labeled_v2}}. This dataset was entirely labeled by hand. It shares the same thematic focus on ESG reports in the fast-food industry as its synthetic counterpart but contains a distinct set of queries with no overlap, offering a purely human-curated evaluation scenario.

\subsection{JinaVDR Benchmark}
\label{app:jinavdr_benchmark}

JinaVDR is a comprehensive benchmark designed to evaluate model performance on retrieving visually complex documents. It stands out by encompassing a wide array of multilingual documents with intricate layouts that combine text with charts, tables, and images, sourced from PDFs, scanned materials, and screenshots. The benchmark spans diverse domains, including historical archives, medical records, and legal texts, reflecting real-world professional use cases. JinaVDR was constructed through a multi-faceted approach: repurposing existing VQA and OCR datasets into retrieval tasks, manually annotating high-quality document-query pairs, and employing a hybrid generation process where a Vision-Language Model (Qwen2-VL-7B-Instruct) creates queries for existing document collections, often followed by human verification.

The benchmark features extensive multilingual coverage. A significant portion of this is derived from the Europeana online collection\footnote{\url{https://www.europeana.eu/en}}, which provides scanned historical documents. This includes datasets of Dutch legal texts (\textbf{europeana-nl-legal}\footnote{\url{https://huggingface.co/datasets/jinaai/europeana-nl-legal}}), Italian historical scans (\textbf{europeana-it-scans}\footnote{\url{https://huggingface.co/datasets/jinaai/europeana-it-scans}}), and news articles in German (\textbf{europeana-de-news}\footnote{\url{https://huggingface.co/datasets/jinaai/europeana-de-news}}) and Spanish (\textbf{europeana-es-news}\footnote{\url{https://huggingface.co/datasets/jinaai/europeana-es-news}}). For these datasets, queries were synthetically generated using Qwen2b, with the Dutch dataset also undergoing manual human review to ensure quality.

JinaVDR also includes datasets created by adapting existing resources or curating new ones for specific languages. For example, \textbf{arabic\_infographicsvqa\_ar}\footnote{\url{https://huggingface.co/datasets/jinaai/arabic_infographicsvqa_ar}} reformats an Arabic VQA dataset for retrieval\footnote{\url{https://huggingface.co/datasets/ahmedheakl/arabic_infographicsvqa}}. Other datasets were built by sourcing documents and generating synthetic queries, such as \textbf{beverages\_catalogue\_ru}\footnote{\url{https://huggingface.co/datasets/jinaai/beverages_catalogue_ru}}, which uses Russian beverage catalogs, and \textbf{hindi-gov-vqa}\footnote{\url{https://huggingface.co/datasets/jinaai/hindi-gov-vqa}}, which is based on public government documents in Hindi.

Finally, the benchmark contains datasets with high-quality, domain-specific query-document pairs. This includes \textbf{shanghai\_master\_plan}\footnote{\url{https://huggingface.co/datasets/jinaai/shanghai_master_plan}}, which pairs pages from Shanghai's official master plan with manually written Chinese queries \cite{shanghai_masterplan_2018}. Another example is \textbf{automobile\_catalogue\_jp}\footnote{\url{https://huggingface.co/datasets/jinaai/automobile_catalogue_jp}}, a specialized dataset using Japanese marketing materials from Toyota\footnote{\url{https://toyota.jp/}} to test retrieval in a commercial context.

\subsection{REAL-MM-RAG Benchmark}
\label{real_mm_rag_benchmark}
REAL-MM-RAG is a high-quality, multimodal retrieval benchmark generated through a fully automated pipeline for query creation, filtering, and label verification. A distinctive feature of this benchmark is its multi-level rephrasing framework, which is designed to rigorously evaluate a model's semantic understanding by testing its robustness to various query phrasings beyond simple keyword matching. The benchmark is structured around two key professional domains: finance and technology, each with datasets in both long-form report and presentation slide formats.

The benchmark's four datasets test retrieval across these different domains and formats. In the financial domain, \textbf{REAL-MM-RAG\_FinReport}\footnote{\url{https://huggingface.co/datasets/ibm-research/REAL-MM-RAG_FinReport}} consists of dense, table-heavy financial reports, while \textbf{REAL-MM-RAG\_FinSlides}\footnote{\url{https://huggingface.co/datasets/ibm-research/REAL-MM-RAG_FinSlides}} contains visually structured quarterly financial presentations. In the technology domain, \textbf{REAL-MM-RAG\_TechReport}\footnote{\url{https://huggingface.co/datasets/ibm-research/REAL-MM-RAG_TechReport}} is composed of text-heavy technical documents on IBM FlashSystem, and \textbf{REAL-MM-RAG\_TechSlides}\footnote{\url{https://huggingface.co/datasets/ibm-research/REAL-MM-RAG_TechSlides}} features presentations on IT and business automation that blend text, visuals, and tables. This structure allows for a nuanced evaluation of retrieval performance on both dense textual information and visually complex layouts.

\subsection{ViDoSeek Benchmark}
\label{app:vidoseek_benchmark}

ViDoSeek\footnote{\url{https://huggingface.co/datasets/autumncc/ViDoSeek}} is a novel benchmark specifically designed to evaluate RAG systems on visually rich documents. It addresses a critical limitation in traditional VQA datasets, which are often not suitable for RAG tasks over large-scale corpora. The core design principle of ViDoSeek is that each query has a unique answer and a specific set of reference pages within the entire document collection, enabling a more precise and realistic evaluation of both retrieval and generation performance. The dataset comprises approximately 1,200 questions spanning diverse domains, content types (text, charts, tables, and complex layouts), and reasoning complexities, including both single-hop and multi-hop questions.

The construction of ViDoSeek involved a meticulous, hybrid human-AI pipeline to ensure data quality and relevance. The process began with the collection of documents, primarily slides, which were filtered to include a rich mix of text, tables, charts, and layouts. Human experts then authored specific queries designed to be answerable only by a particular document. A crucial step was the automated quality review, where an LLM first filtered out general queries, and a VLM subsequently verified that the remaining queries had unique answers within the entire corpus by checking against retrieved candidates. Finally, queries that failed this check were refined by a VLM-based agent. The resulting dataset is composed of newly created data and refined queries from the existing SlideVQA dataset \cite{tanaka2023slidevqa}, with a significant portion of questions targeting challenging "Layout" based content.

\subsection{MMLongBench-Doc Benchmark}
\label{app:mmlongbench_doc_benchmark}

MMLongBench-Doc\footnote{\url{https://huggingface.co/datasets/yubo2333/MMLongBench-Doc}} is a benchmark designed to evaluate the long-context, multi-modal document understanding capabilities of LVLMs. It addresses the limitations of previous datasets that primarily focus on single-page documents by providing a challenging testbed of lengthy, information-dense documents. The benchmark consists of 1,082 expert-annotated questions across 135 PDF documents, which average 47.5 pages and over 21,000 tokens each. Questions require evidence from diverse sources, including text, tables, charts, images, and layout structures, and are distributed across various locations within the documents to rigorously test models' localization and comprehension abilities.

The construction of MMLongBench-Doc emphasizes diversity and quality. Documents were sourced from existing datasets and newly collected materials, spanning seven different domains. A key feature of the benchmark is its diverse question types, designed to test distinct capabilities: 45.7\% are single-page questions to assess information localization, a significant 33.7\% are cross-page questions that require multi-page reasoning, and 20.6\% are unanswerable questions designed to measure model hallucination. To ensure high annotation quality, a rigorous, three-stage, semi-automatic quality control process was employed, involving filtering for document relevance, annotator self-reflection guided by model feedback, and peer cross-checking. This meticulous process makes MMLongBench-Doc a robust and challenging benchmark for long-context document understanding.

\clearpage
\section{Baseline Details}
\label{app:baseline_details}
This section details the implementation logic and hyperparameter configurations for the baseline methods evaluated in our experiments. For each method, we conducted an empirical search over the specified hyperparameter range to report the most representative performance trade-offs.

\subsection{Pruning-based Methods}

\paragraph{$\blacktriangleright$ Random}
\begin{itemize}[leftmargin=1.5em, itemsep=0em]
    \item \textit{Implementation Logic:} This naive baseline discards a specified fraction of patch embeddings uniformly at random for each document. It serves as a lower bound to assess the impact of uninformed pruning. To avoid empty sets, at least one patch is always retained.
    \item \textit{Hyperparameters:}
    \begin{itemize}[itemsep=0em]
        \item \texttt{pruning\_ratio}: The fraction of embeddings to discard. Selection Range: \texttt{\{0.1, 0.3, 0.5, 0.7, 0.9\}}.
    \end{itemize}
\end{itemize}

\paragraph{$\blacktriangleright$ Attention-plus-Similarity}
\begin{itemize}[leftmargin=1.5em, itemsep=0em]
    \item \textit{Implementation Logic:} This adaptive baseline calculates a composite score for each patch as a weighted sum of its importance (attention from the \texttt{[EOS]} token) and representativeness (cosine similarity to the \texttt{[EOS]} embedding). It then applies a dynamic threshold ($\mu + k \cdot \sigma$) to these composite scores to prune patches.
    \item \textit{Hyperparameters:}
    \begin{itemize}[itemsep=0em]
        \item Adaptation Factor (\texttt{k}): Modulates the pruning threshold's strictness. Selection Range: \texttt{\{-0.5, -0.25, 0, 0.25, 0.5, 1\}}.
        \item Weighting Factor (\texttt{alpha}): Balances the attention and similarity components. Selection Range: \texttt{\{0.1, 0.3, 0.5, 0.7, 0.9\}}.
    \end{itemize}
\end{itemize}

\paragraph{$\blacktriangleright$ Pivot-Threshold}
\begin{itemize}[leftmargin=1.5em, itemsep=0em]
    \item \textit{Implementation Logic:} This two-stage adaptive method first identifies an ``important set'' of patches using an attention-based threshold ($\mu + k \cdot \sigma$). It then selects \texttt{num\_pivots} from this set and prunes other patches in the set if they are too similar to any pivot, determined by a second adaptive threshold on similarity scores.
    \item \textit{Hyperparameters:}
    \begin{itemize}[itemsep=0em]
        \item Adaptation Factor (\texttt{k}): Controls the initial importance filtering. Selection Range: \texttt{\{-0.5, -0.25, 0, 0.25, 0.5, 1\}}.
        \item De-duplication Factor (\texttt{k\_dup}): Controls the similarity-based pruning. Selection Range: \texttt{\{-0.5, -0.25, 0, 0.25, 0.5, 1\}}.
        \item \texttt{num\_pivots}: The number of pivot tokens for de-duplication. Selection Range: \texttt{\{5, 10, 15, 20\}}.
    \end{itemize}
\end{itemize}

\paragraph{$\blacktriangleright$ DocPruner}
\begin{itemize}[leftmargin=1.5em, itemsep=0em]
    \item \textit{Implementation Logic:} This method leverages the attention scores from a global token (e.g., \texttt{[EOS]}) to all other patch tokens within a document. It computes a document-specific threshold using the mean and standard deviation of these scores ($\mu + k \cdot \sigma$) and discards any patch whose attention score falls below this dynamic threshold.
    \item \textit{Hyperparameters:}
    \begin{itemize}[itemsep=0em]
        \item Adaptation Factor (\texttt{k}): Controls the strictness of the pruning threshold. Selection Range: \texttt{\{-0.5, -0.25, 0, 0.25, 0.5, 1\}}.
    \end{itemize}
\end{itemize}

\subsection{Merging-based Methods}

\paragraph{$\blacktriangleright$ Sem-Cluster}
\begin{itemize}[leftmargin=1.5em, itemsep=0em]
    \item \textit{Implementation Logic:} This method groups embeddings semantically. It performs hierarchical agglomerative clustering on the set of patch embeddings using cosine distance and Ward's linkage. The number of target clusters is determined by the \texttt{merging\_factor}. The final representation consists of the centroids (mean vectors) of each cluster.
    \item \textit{Hyperparameters:}
    \begin{itemize}[itemsep=0em]
        \item \texttt{merging\_factor}: The ratio by which the number of embeddings is reduced. Selection Range: \texttt{\{2, 4, 9, 16, 25\}}.
    \end{itemize}
\end{itemize}

\paragraph{$\blacktriangleright$ 1D-Pooling}
\begin{itemize}[leftmargin=1.5em, itemsep=0em]
    \item \textit{Implementation Logic:} This strategy applies average pooling along the sequence of patch embeddings. It partitions the sequence into non-overlapping windows of size \texttt{merging\_factor} and computes the mean of embeddings within each window.
    \item \textit{Hyperparameters:}
    \begin{itemize}[itemsep=0em]
        \item \texttt{merging\_factor}: The size of the pooling window. Selection Range: \texttt{\{2, 4, 9, 16, 25\}}.
    \end{itemize}
\end{itemize}

\paragraph{$\blacktriangleright$ 2D-Pooling}
\begin{itemize}[leftmargin=1.5em, itemsep=0em]
    \item \textit{Implementation Logic:} This method leverages the spatial layout of patches. It arranges embeddings into a 2D grid and applies 2D average pooling. The \texttt{merging\_factor} must be a perfect square, defining the pooling kernel's area (e.g., a factor of 4 implies a 2x2 kernel).
    \item \textit{Hyperparameters:}
    \begin{itemize}[itemsep=0em]
        \item \texttt{merging\_factor}: The area of the 2D pooling window. Selection Range: \texttt{\{4, 9, 16, 25\}}.
    \end{itemize}
\end{itemize}

\clearpage
\section{More Experimental Result \& Analysis}
\label{app:more_experimental_analysis}

\subsection{Empirical Experimental Results}
\label{app:empirical_exp_results}

We record all the emprical experimental results on all benchamrks from Table \ref{tab:vidore_v1_colqwen} to Table \ref{tab:mmlongbench_jina} for your reference. Due to the space limit, we only maintain two decimal places for all data cells.

\subsection{VDR Performance Comparison}
\label{app:vdr_performance_comparison}
You can refer to Figures \ref{fig:vidore_v2_performance_all}, \ref{fig:vidoseek_performance_all}, and \ref{fig:mmlongbench_performance_all} for the performance comparisons on ViDoRe-V2 \cite{macé2025vidorebenchmarkv2raising}, ViDoSeek \cite{wang2025vidorag}, and MMLongBench-Doc \cite{ma2024mmlongbench}, respectively.

\begin{figure*}[!t]
\centering
\includegraphics[width=\linewidth]{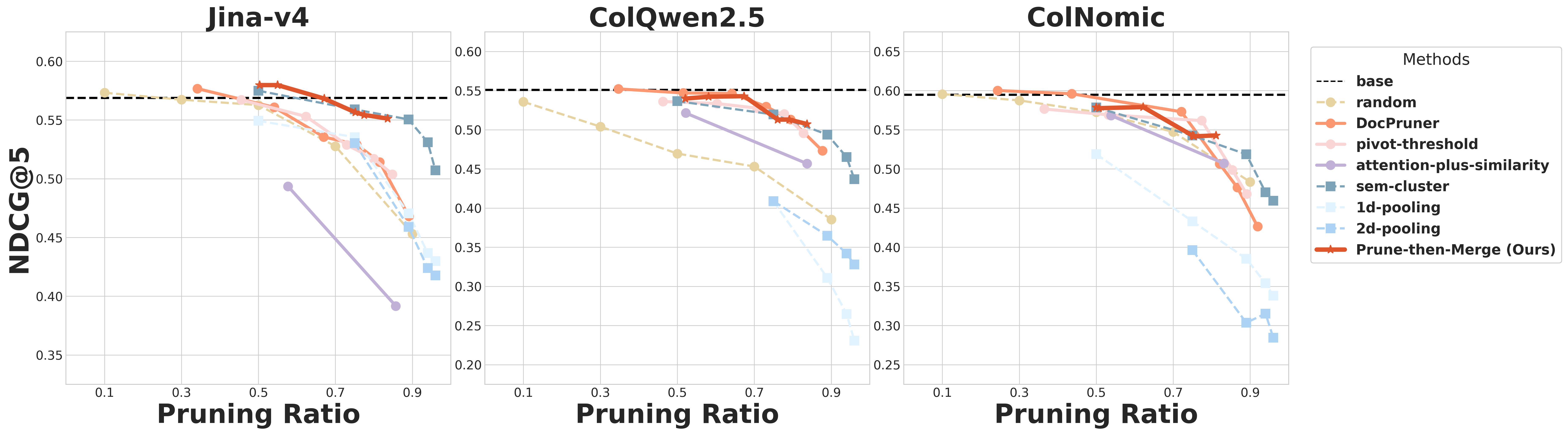}
\vspace{-2em}
\caption{Performance comparison (nDCG@5) between \method and baselines on ViDoRe-V2 \citep{macé2025vidorebenchmarkv2raising} across  Jina-v4 (\textbf{\textit{Left}}), ColQwen2.5 (\textbf{\textit{Middle}}), and ColNomic (\textbf{\textit{Right}}). \textit{solid lines} denote adaptive methods, whereas \textit{dashed lines} denote non-adaptive ones; \textit{circular nodes} represent pruning methods, whereas \textit{square nodes }represent merging ones.}
\label{fig:vidore_v2_performance_all}
\end{figure*}

\begin{figure*}[!t]
\centering
\includegraphics[width=\linewidth]{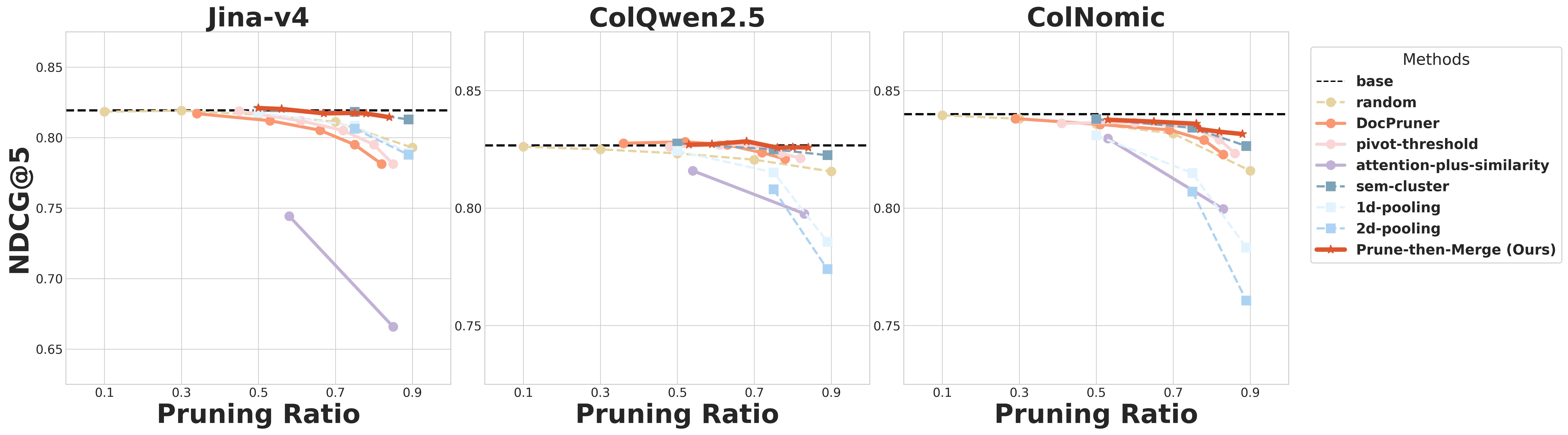}
\vspace{-2em}
\caption{Performance comparison (nDCG@5) between \method and baselines on ViDoSeek \citep{wang2025vidorag} across  Jina-v4 (\textbf{\textit{Left}}), ColQwen2.5 (\textbf{\textit{Middle}}), and ColNomic (\textbf{\textit{Right}}). \textit{solid lines} denote adaptive methods, whereas \textit{dashed lines} denote non-adaptive ones; \textit{circular nodes} represent pruning methods, whereas \textit{square nodes }represent merging ones.}
\label{fig:vidoseek_performance_all}
\end{figure*}

\begin{figure*}[!t]
\centering
\includegraphics[width=\linewidth]{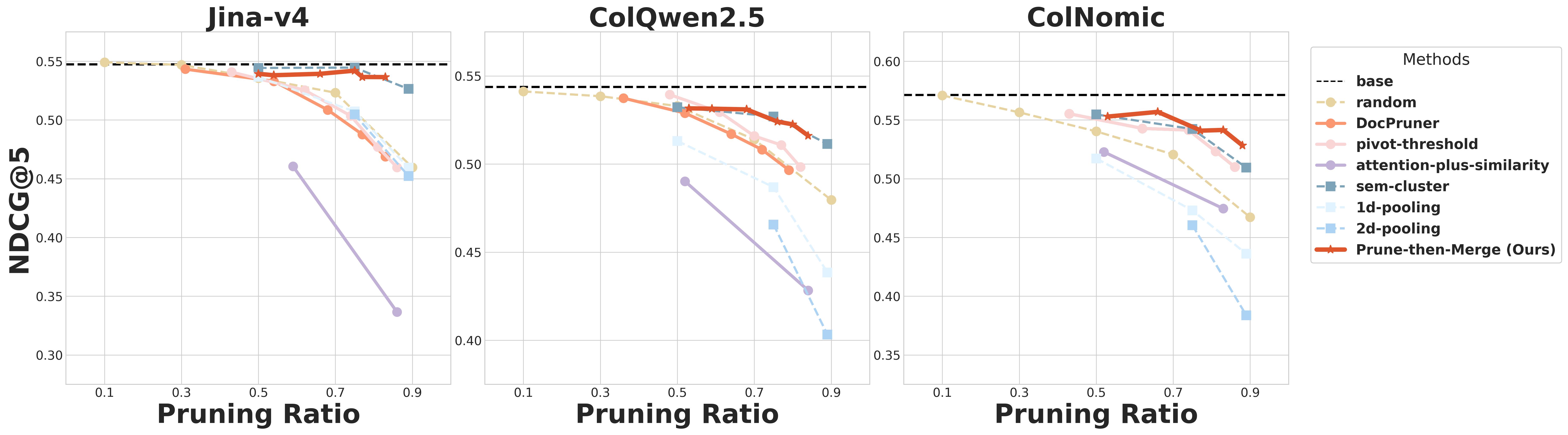}
\vspace{-2em}
\caption{Performance comparison (nDCG@5) between \method and baselines on MMLongBench-Doc \citep{ma2024mmlongbench} across  Jina-v4 (\textbf{\textit{Left}}), ColQwen2.5 (\textbf{\textit{Middle}}), and ColNomic (\textbf{\textit{Right}}). \textit{solid lines} denote adaptive methods, whereas \textit{dashed lines} denote non-adaptive ones; \textit{circular nodes} represent pruning methods, whereas \textit{square nodes }represent merging ones.}
\label{fig:mmlongbench_performance_all}
\end{figure*}

\subsection{Generalization to Multilingual Scenarios}
\label{app:generalization_multilingual}

In this section, we continue to discuss the observations we obtain from Figure \ref{fig:jinavdr_performance_all} and Tables \ref{tab:jinavdr_colqwen}, \ref{tab:jinavdr_colnomic}, and \ref{tab:jinavdr_jina}.

Interestingly, the performance advantage of \method is particularly pronounced on non-Latin script languages such as Japanese and Chinese, which often feature more complex layouts and character densities.
For ColNomic, our method compresses Japanese documents by 89\% while achieving an nDCG@5 of 0.82, whereas the strong merging baseline Sem-Cluster drops to 0.76.
Similarly, for Chinese documents, our method maintains a near-lossless 0.95 nDCG@5 at 89\% compression, while Sem-Cluster degrades to 0.94.
This suggests that the visual and semantic redundancy in documents with complex character sets and layouts is more effectively handled by our two-stage "refine-then-compress" approach, which can untangle noisy patches from meaningful clusters before aggregation.

Furthermore, \method demonstrates superior resilience on languages where the base model's performance is inherently weaker, such as on the Hindi and Dutch datasets.
On the Hindi dataset, where the baseline nDCG@5 is a low 0.47, DocPruner's performance collapses to 0.40 at an 83\% compression rate, whereas our framework sustains a much higher score of 0.46.
A similar pattern is observed on the Dutch dataset, where our method maintains the baseline score of 0.68 up to 82\% compression, while DocPruner drops to 0.56.
This indicates that when the initial signal is weak, aggressive pruning is overly destructive, whereas our hybrid approach better preserves the faint but critical information by first filtering salient signals and then carefully summarizing them, thus preventing catastrophic performance degradation.

\subsection{Generalization to Complex Settings}
\label{app:generalization_complex setting}

In this section, we continue to discuss the observations we obtain from Figure \ref{fig:realmmrag_performance_all} and Tables \ref{tab:realmmrag_colqwen}, \ref{tab:realmmrag_colnomic}, and \ref{tab:realmmrag_jina}.

Moreover, the performance gap between \method and the strong Sem-Cluster baseline appears to be model-dependent, highlighting our framework's unique synergy with certain embedding characteristics.
While our method performs on par with Sem-Cluster when using Jina-v4 and ColNomic, it establishes a clear advantage with ColQwen2.5, as seen in Figure~\ref{fig:realmmrag_performance_all} (Middle).
Specifically, at an 84\% compression rate, \method (0.65 nDCG@5) is notably superior to Sem-Cluster (0.61 nDCG@5).
This suggests that the patch embeddings from ColQwen2.5 may possess a more distinct separation between "noise" and "signal" vectors, making the initial pruning stage of our framework particularly impactful.
By providing a cleaner, pre-filtered set of embeddings to the merging stage, our method capitalizes on this characteristic to produce a more faithful final representation than a merge-only approach operating on raw, unfiltered vectors.

\subsection{Variant Analysis}
\label{app:variant_analysis}

Adaptive, document-specific pruning is decisively superior to non-adaptive strategies that rely on fixed ratios or global thresholds.
The performance curves in Figure~\ref{fig:variant_study_vidore_v2_jina} clearly show that both adaptive methods (\method and \texttt{attention-threshold-adap}) establish a significant performance margin over all non-adaptive variants.
At a moderate pruning rate of approximately 55\%, the adaptive methods achieve an nDCG@5 of 0.56, while the best non-adaptive method, \texttt{attention-threshold-nfp}, lags behind at 0.52, and the simpler \texttt{attention-ratio} variant drops to 0.51.
This performance gap arises because documents have heterogeneous content density; a fixed threshold or ratio is too aggressive for information-rich pages and too lenient for sparse ones, whereas adaptive thresholding tailors the pruning level to each document's unique statistical properties, preserving critical information more effectively.

The performance advantage of our hybrid framework is most pronounced on datasets with higher semantic complexity, such as the human-annotated \texttt{esg\_human\_labeled\_v2}.
A detailed comparison in Appendix Figures~\ref{fig:variant_vidore_v2_jina_esg}-\ref{fig:variant_vidore_v2_jina_esg_human} reveals that while \method consistently leads across all datasets, the performance gap widens on this particular one.
As illustrated in Appendix Figure~\ref{fig:variant_vidore_v2_jina_esg_human}, at an \textasciitilde80\% compression rate, \method scores 0.60 nDCG@5, decisively outperforming the adaptive-pruning-only variant which plummets to 0.54.
In contrast, on the more synthetically generated \texttt{economics\_reports\_v2} dataset (Appendix Figure~\ref{fig:variant_vidore_v2_jina_economic}), this gap is much narrower at the same compression level (0.51 vs. 0.49).
We infer that human-generated queries, being less extractive and more conceptual, demand a more robust semantic representation, which is better provided by our hybrid approach that distills concepts into potent centroids rather than merely discarding patch embeddings.

\begin{figure*}[!t]
\centering
\includegraphics[width=0.8\linewidth]{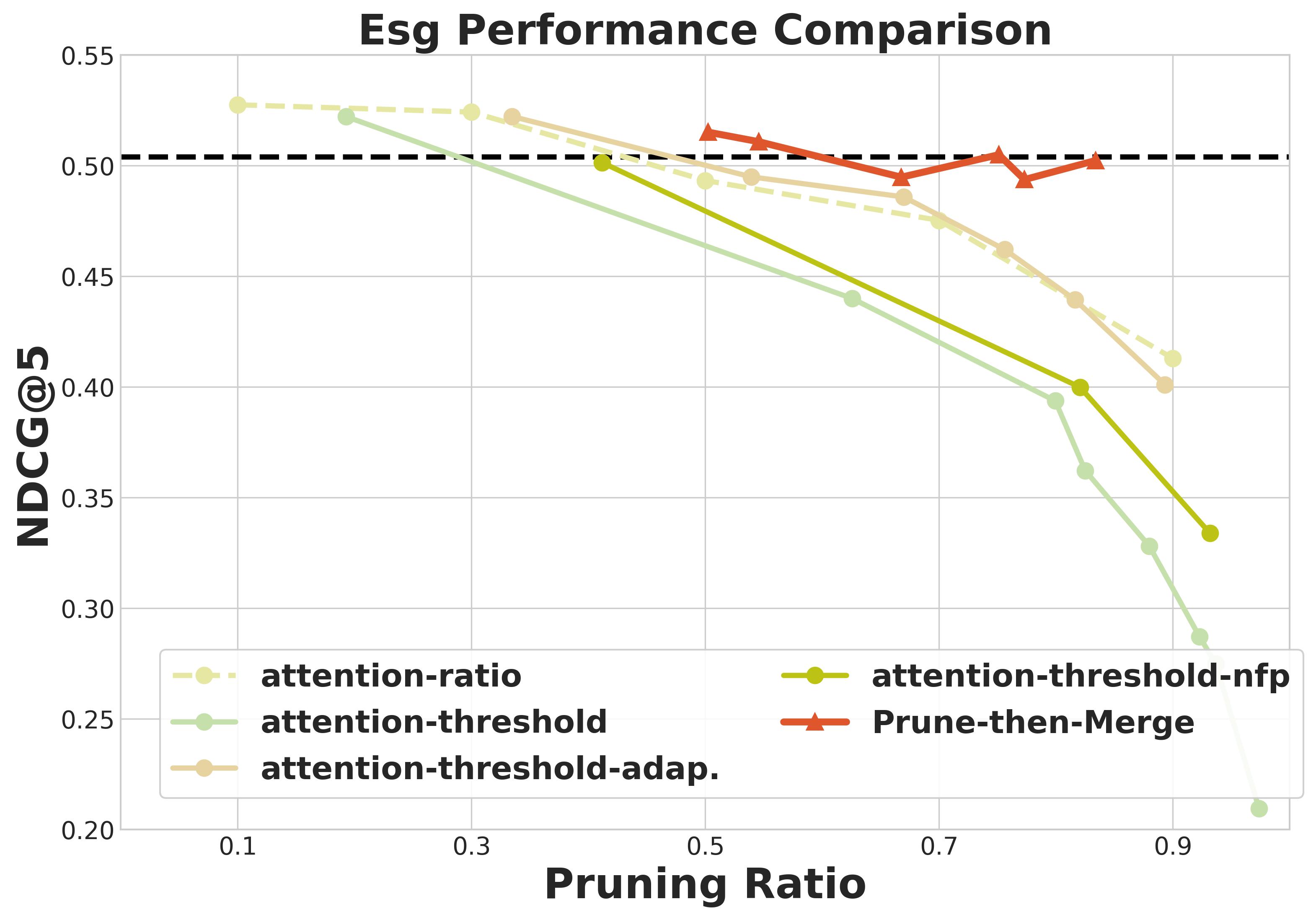}
\caption{Variant comparison of Jina-v4 on \textbf{ESG} dataset of ViDoRe-V2. \textit{solid lines} denote adaptive methods, whereas \textit{dashed lines} denote non-adaptive ones.}
\label{fig:variant_vidore_v2_jina_esg}
\end{figure*}

\begin{figure*}[!t]
\centering
\includegraphics[width=0.8\linewidth]{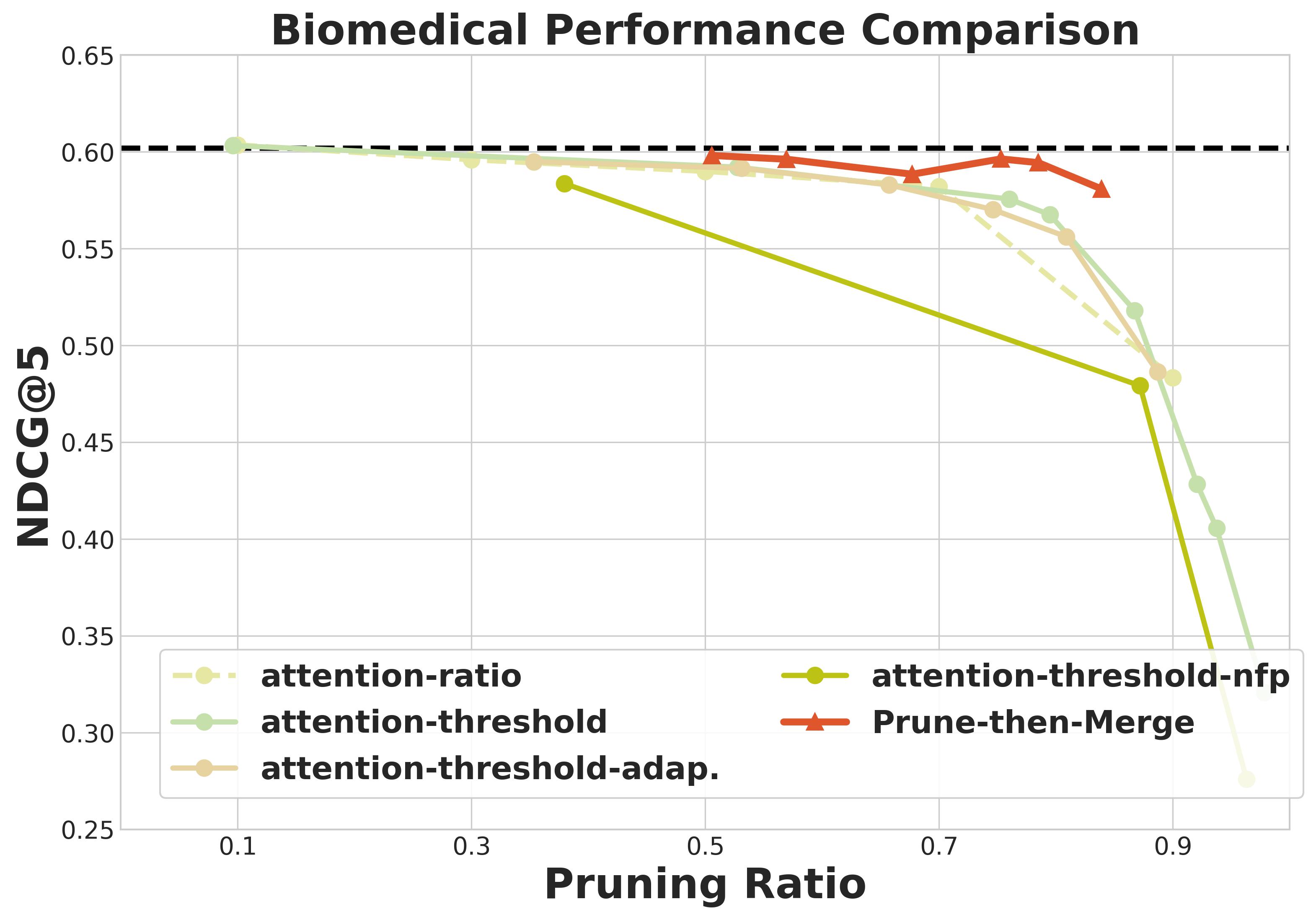}
\caption{Variant comparison of Jina-v4 on \textbf{Biomedical} dataset of ViDoRe-V2. \textit{solid lines} denote adaptive methods, whereas \textit{dashed lines} denote non-adaptive ones.}
\label{fig:variant_vidore_v2_jina_biomedical}
\end{figure*}

\begin{figure*}[!t]
\centering
\includegraphics[width=0.8\linewidth]{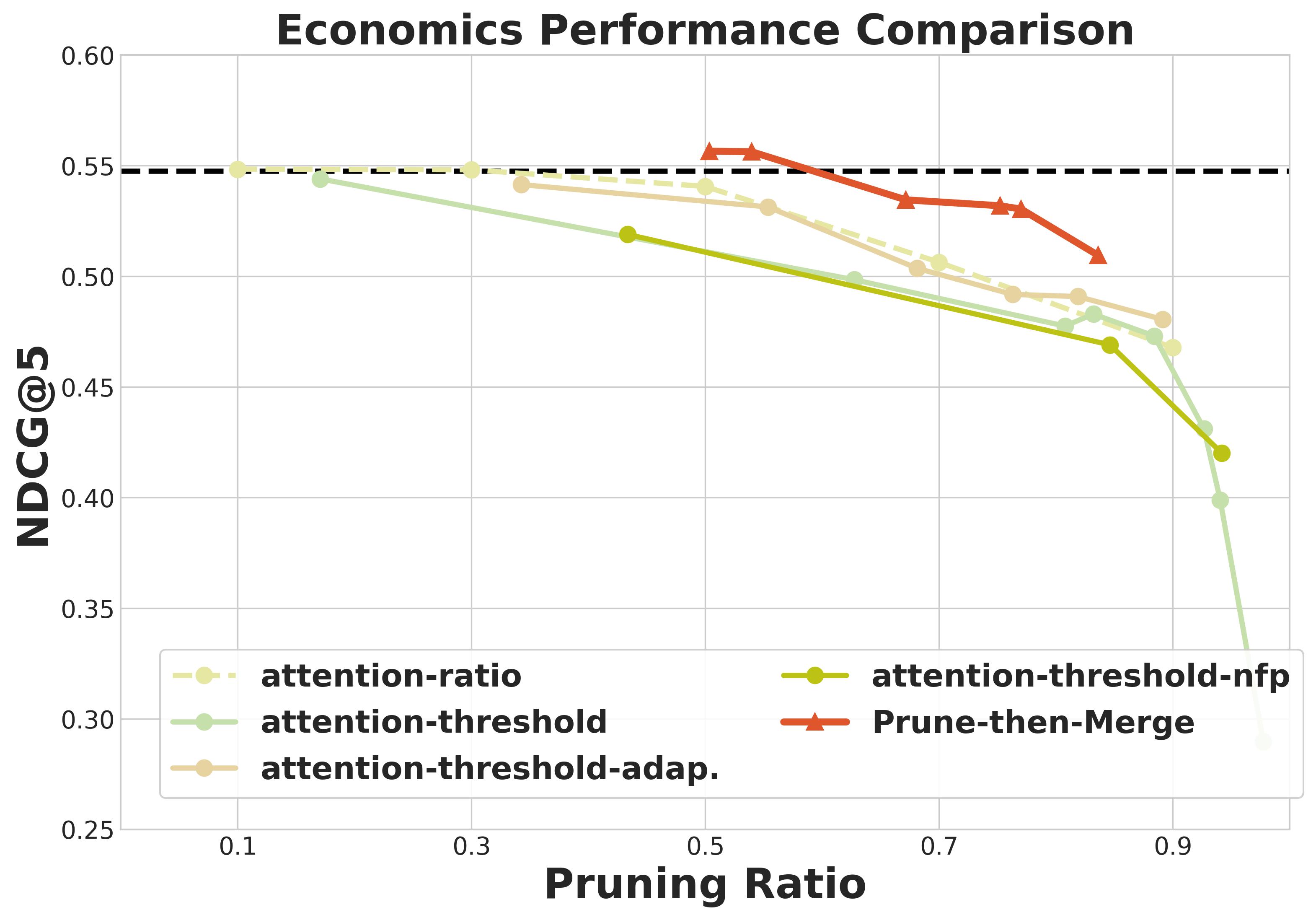}
\caption{Variant comparison of Jina-v4 on \textbf{Economic} dataset of ViDoRe-V2. \textit{solid lines} denote adaptive methods, whereas \textit{dashed lines} denote non-adaptive ones.}
\label{fig:variant_vidore_v2_jina_economic}
\end{figure*}

\begin{figure*}[!t]
\centering
\includegraphics[width=0.8\linewidth]{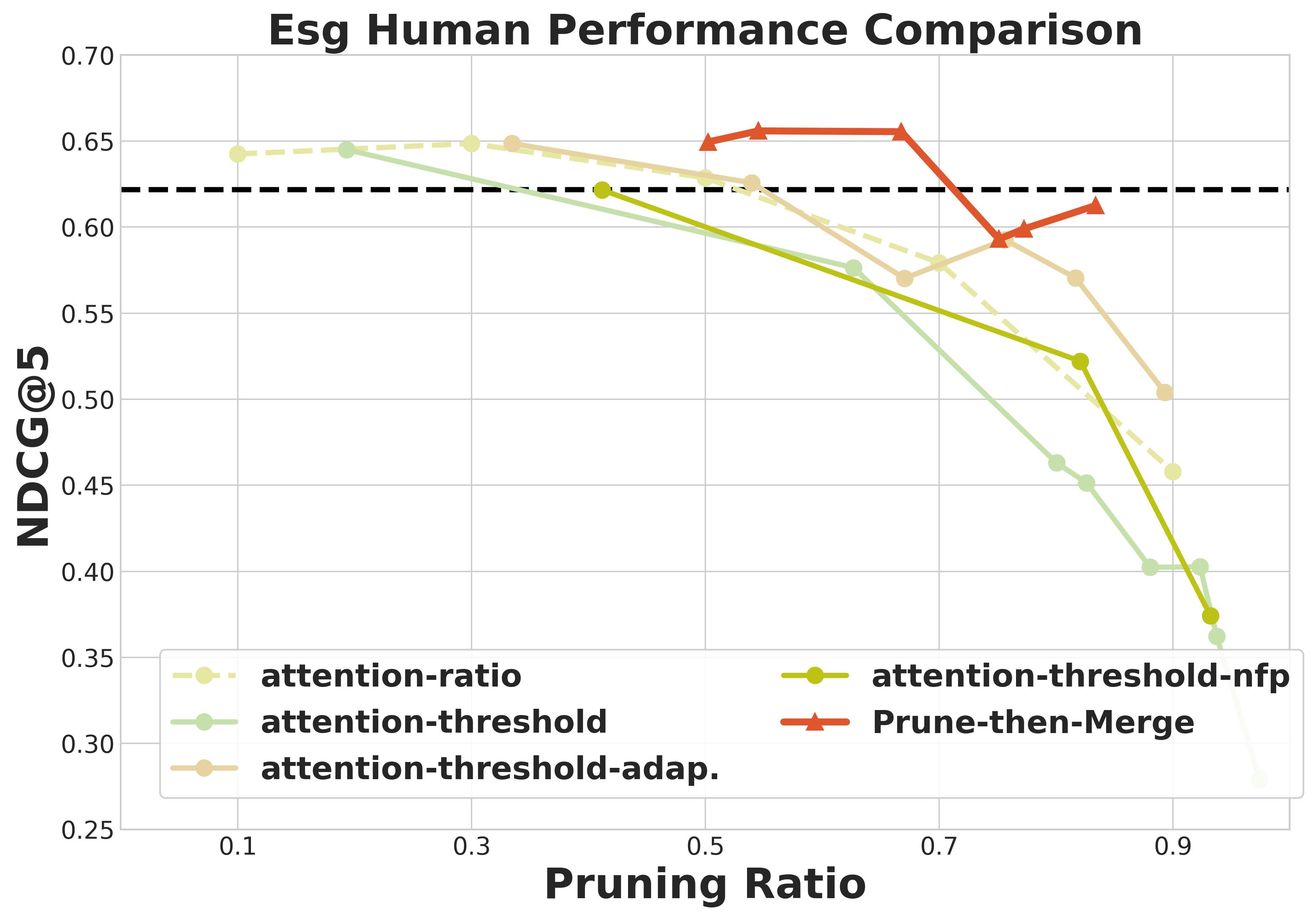}
\caption{Variant comparison of Jina-v4 on \textbf{ESG Human} dataset of ViDoRe-V2. \textit{solid lines} denote adaptive methods, whereas \textit{dashed lines} denote non-adaptive ones.}
\label{fig:variant_vidore_v2_jina_esg_human}
\end{figure*}

\subsection{Efficiency Analysis}
\label{app:efficiency_analysis}

See Table \ref{tab:efficiency_analysis_vidore_v1_jina_full} for details of efficiency analysis of the chosen Jina-v4 on ViDoRe-V1.

\begin{table*}[h!]
 \centering
  \caption{
  Comparison of our \method with base models on ViDoRe-V1, in terms of \textbf{performance}, \textbf{storage}, and \textbf{latency}. The table presents absolute values and the relative change ($\Delta$). ({We set adaptation factor as -0.75 and merging factor as 4; \color{RedOrange}orange} denotes a better outcome and {\color{green(pigment)}green} denotes a worse outcome.)
  }
  \label{tab:efficiency_analysis_vidore_v1_jina_full} 
  \vspace{-0.8em}
  \renewcommand\tabcolsep{6pt} 
  \renewcommand\arraystretch{1.1}
  \footnotesize 
  \begin{tabular}{ll|rrrr} 
    \Xhline{1.2pt}
    \rowcolor{CadetBlue!20} 
    & & \textbf{ColQwen2.5} & \textbf{ColNomic} & \textbf{JinaV4} & \textbf{Overall}\\ 
    \Xhline{1pt}
    
    \multirow{3}{*}{\begin{tabular}{@{}c@{}}\textbf{Performance} (nDCG@5)\end{tabular}} 
    & base & 0.8795 & 0.8999 & 0.8777 & 0.8857 \\
    & \method & 0.8771 & 0.8953 & 0.8727 & 0.8817 \\
    \rowcolor{green(pigment)!30}
    & $\Delta$ & 0.27\% & 0.51\% & 0.57\% & 0.45\% \\
    \Xhline{1pt}

    \multirow{3}{*}{\textbf{Storage}}
    & base & 1.0000 & 1.0000 & 1.0000 & 1.0000 \\
    & \method & 0.4112 & 0.4784 & 0.4723 & 0.4500 \\
    \rowcolor{RedOrange!40}
    & $\Delta$ & 58.88\% & 52.16\% & 52.77\% & 54.60\% \\
    \Xhline{1pt}

    \multirow{3}{*}{\textbf{Latency}}
    & base & 0.41 & 0.45 & 0.51 & 0.46 \\
    & \method & 0.64 & 0.68 & 0.75 & 0.69 \\
    \rowcolor{green(pigment)!30}
    & $\Delta$ & -56.10\% & -51.11\% & -47.06\% & -51.42\% \\
    \Xhline{1.2pt}
  \end{tabular}
  \vspace{-0.7em}
\end{table*}

\begin{table*}[ht]
\centering
\caption{Experimental results of ColQwen on ViDoRe-V1.}
\resizebox{1.0\linewidth}{!}{
\begin{tabular}{llccllcccccccccccccccccccc}
\toprule
\textbf{Type} & \textbf{Methods} & \textbf{Adp.} & \textbf{Para} & \multicolumn{2}{c}{\textbf{all}} & \multicolumn{2}{c}{\textbf{arxqa}} & \multicolumn{2}{c}{\textbf{docqa}} & \multicolumn{2}{c}{\textbf{infovqa}} & \multicolumn{2}{c}{\textbf{tabfquad}} & \multicolumn{2}{c}{\textbf{tatdqa}} & \multicolumn{2}{c}{\textbf{shiftproject}} & \multicolumn{2}{c}{\textbf{synth\_ai}} & \multicolumn{2}{c}{\textbf{synth\_energy}} & \multicolumn{2}{c}{\textbf{synth\_gov}} & \multicolumn{2}{c}{\textbf{synth\_health}} \\
\cmidrule(lr){5-6} \cmidrule(lr){7-8} \cmidrule(lr){9-10} \cmidrule(lr){11-12} \cmidrule(lr){13-14} \cmidrule(lr){15-16} \cmidrule(lr){17-18} \cmidrule(lr){19-20} \cmidrule(lr){21-22} \cmidrule(lr){23-24} \cmidrule(lr){25-26}
& & & & ndcg & prun\% & ndcg & prun\% & ndcg & prun\% & ndcg & prun\% & ndcg & prun\% & ndcg & prun\% & ndcg & prun\% & ndcg & prun\% & ndcg & prun\% & ndcg & prun\% & ndcg & prun\% \\
\midrule
pruning & base &  & - & 0.87 & - & 0.86 & - & 0.60 & - & 0.92 & - & 0.90 & - & 0.76 & - & 0.86 & - & 0.98 & - & 0.95 & - & 0.96 & - & 0.96 & - \\
\rowcolor{lightgray} pruning & random &  & ratio=0.1 & 0.87 & 0.10 & 0.86 & 0.10 & 0.60 & 0.10 & 0.92 & 0.10 & 0.91 & 0.10 & 0.76 & 0.10 & 0.85 & 0.10 & 0.98 & 0.10 & 0.95 & 0.10 & 0.95 & 0.10 & 0.97 & 0.10 \\
pruning & random &  & ratio=0.3 & 0.87 & 0.30 & 0.86 & 0.30 & 0.60 & 0.30 & 0.91 & 0.30 & 0.91 & 0.30 & 0.76 & 0.30 & 0.84 & 0.30 & 0.98 & 0.30 & 0.95 & 0.30 & 0.96 & 0.30 & 0.97 & 0.30 \\
\rowcolor{lightgray} pruning & random &  & ratio=0.5 & 0.87 & 0.50 & 0.86 & 0.50 & 0.60 & 0.50 & 0.91 & 0.50 & 0.91 & 0.50 & 0.76 & 0.50 & 0.82 & 0.50 & 0.98 & 0.50 & 0.96 & 0.50 & 0.95 & 0.50 & 0.97 & 0.50 \\
pruning & random &  & ratio=0.7 & 0.86 & 0.70 & 0.85 & 0.70 & 0.58 & 0.70 & 0.91 & 0.70 & 0.91 & 0.70 & 0.75 & 0.70 & 0.79 & 0.70 & 0.97 & 0.70 & 0.95 & 0.70 & 0.96 & 0.70 & 0.97 & 0.70 \\
\bottomrule
\end{tabular}
}\label{tab:vidore_v1_colqwen}
\end{table*}

\begin{table*}[ht]
\centering
\caption{Experimental results of ColNomic on ViDoRe-V1.}
\resizebox{1.0\linewidth}{!}{
\begin{tabular}{llccllcccccccccccccccccccc}
\toprule
\textbf{Type} & \textbf{Methods} & \textbf{Adp.} & \textbf{Para} & \multicolumn{2}{c}{\textbf{all}} & \multicolumn{2}{c}{\textbf{arxqa}} & \multicolumn{2}{c}{\textbf{docqa}} & \multicolumn{2}{c}{\textbf{infovqa}} & \multicolumn{2}{c}{\textbf{tabfquad}} & \multicolumn{2}{c}{\textbf{tatdqa}} & \multicolumn{2}{c}{\textbf{shiftproject}} & \multicolumn{2}{c}{\textbf{synth\_ai}} & \multicolumn{2}{c}{\textbf{synth\_energy}} & \multicolumn{2}{c}{\textbf{synth\_gov}} & \multicolumn{2}{c}{\textbf{synth\_health}} \\
\cmidrule(lr){5-6} \cmidrule(lr){7-8} \cmidrule(lr){9-10} \cmidrule(lr){11-12} \cmidrule(lr){13-14} \cmidrule(lr){15-16} \cmidrule(lr){17-18} \cmidrule(lr){19-20} \cmidrule(lr){21-22} \cmidrule(lr){23-24} \cmidrule(lr){25-26}
& & & & ndcg & prun\% & ndcg & prun\% & ndcg & prun\% & ndcg & prun\% & ndcg & prun\% & ndcg & prun\% & ndcg & prun\% & ndcg & prun\% & ndcg & prun\% & ndcg & prun\% & ndcg & prun\% \\
\midrule
pruning & base &  & - & 0.89 & - & 0.87 & - & 0.60 & - & 0.93 & - & 0.94 & - & 0.82 & - & 0.91 & - & 0.98 & - & 0.97 & - & 0.97 & - & 0.98 & - \\
\rowcolor{lightgray} pruning & random &  & ratio=0.1 & 0.89 & 0.10 & 0.87 & 0.10 & 0.60 & 0.10 & 0.92 & 0.10 & 0.94 & 0.10 & 0.82 & 0.10 & 0.90 & 0.10 & 0.98 & 0.10 & 0.96 & 0.10 & 0.97 & 0.10 & 0.98 & 0.10 \\
pruning & random &  & ratio=0.3 & 0.89 & 0.30 & 0.87 & 0.30 & 0.60 & 0.30 & 0.92 & 0.30 & 0.94 & 0.30 & 0.81 & 0.30 & 0.90 & 0.30 & 0.98 & 0.30 & 0.96 & 0.30 & 0.97 & 0.30 & 0.99 & 0.30 \\
\rowcolor{lightgray} pruning & random &  & ratio=0.5 & 0.89 & 0.50 & 0.87 & 0.50 & 0.59 & 0.50 & 0.92 & 0.50 & 0.94 & 0.50 & 0.81 & 0.50 & 0.88 & 0.50 & 0.98 & 0.50 & 0.97 & 0.50 & 0.97 & 0.50 & 0.99 & 0.50 \\
pruning & random &  & ratio=0.7 & 0.88 & 0.70 & 0.86 & 0.70 & 0.58 & 0.70 & 0.91 & 0.70 & 0.94 & 0.70 & 0.80 & 0.70 & 0.85 & 0.70 & 0.98 & 0.70 & 0.96 & 0.70 & 0.97 & 0.70 & 0.98 & 0.70 \\
\bottomrule
\end{tabular}
}\label{tab:vidore_v1_colnomic}
\end{table*}

\begin{table*}[ht]
\centering
\caption{Experimental results of Jina-v4 on ViDoRe-V1.}
\resizebox{1.0\linewidth}{!}{
\begin{tabular}{llccllcccccccccccccccccccc}
\toprule
\textbf{Type} & \textbf{Methods} & \textbf{Adp.} & \textbf{Para} & \multicolumn{2}{c}{\textbf{all}} & \multicolumn{2}{c}{\textbf{arxqa}} & \multicolumn{2}{c}{\textbf{docqa}} & \multicolumn{2}{c}{\textbf{infovqa}} & \multicolumn{2}{c}{\textbf{tabfquad}} & \multicolumn{2}{c}{\textbf{tatdqa}} & \multicolumn{2}{c}{\textbf{shiftproject}} & \multicolumn{2}{c}{\textbf{synth\_ai}} & \multicolumn{2}{c}{\textbf{synth\_energy}} & \multicolumn{2}{c}{\textbf{synth\_gov}} & \multicolumn{2}{c}{\textbf{synth\_health}} \\
\cmidrule(lr){5-6} \cmidrule(lr){7-8} \cmidrule(lr){9-10} \cmidrule(lr){11-12} \cmidrule(lr){13-14} \cmidrule(lr){15-16} \cmidrule(lr){17-18} \cmidrule(lr){19-20} \cmidrule(lr){21-22} \cmidrule(lr){23-24} \cmidrule(lr){25-26}
& & & & ndcg & prun\% & ndcg & prun\% & ndcg & prun\% & ndcg & prun\% & ndcg & prun\% & ndcg & prun\% & ndcg & prun\% & ndcg & prun\% & ndcg & prun\% & ndcg & prun\% & ndcg & prun\% \\
\midrule
pruning & base &  & - & 0.87 & - & 0.86 & - & 0.58 & - & 0.90 & - & 0.92 & - & 0.72 & - & 0.88 & - & 0.98 & - & 0.95 & - & 0.96 & - & 0.97 & - \\
\rowcolor{lightgray} pruning & random &  & ratio=0.1 & 0.87 & 0.10 & 0.86 & 0.10 & 0.58 & 0.10 & 0.90 & 0.10 & 0.92 & 0.10 & 0.72 & 0.10 & 0.88 & 0.10 & 0.98 & 0.10 & 0.96 & 0.10 & 0.96 & 0.10 & 0.97 & 0.10 \\
pruning & random &  & ratio=0.3 & 0.87 & 0.30 & 0.86 & 0.30 & 0.58 & 0.30 & 0.89 & 0.30 & 0.92 & 0.30 & 0.72 & 0.30 & 0.88 & 0.30 & 0.98 & 0.30 & 0.95 & 0.30 & 0.96 & 0.30 & 0.97 & 0.30 \\
\rowcolor{lightgray} pruning & random &  & ratio=0.5 & 0.87 & 0.50 & 0.85 & 0.50 & 0.57 & 0.50 & 0.89 & 0.50 & 0.92 & 0.50 & 0.72 & 0.50 & 0.86 & 0.50 & 0.98 & 0.50 & 0.95 & 0.50 & 0.96 & 0.50 & 0.96 & 0.50 \\
pruning & random &  & ratio=0.7 & 0.85 & 0.70 & 0.85 & 0.70 & 0.54 & 0.70 & 0.87 & 0.70 & 0.92 & 0.70 & 0.70 & 0.70 & 0.84 & 0.70 & 0.97 & 0.70 & 0.93 & 0.70 & 0.97 & 0.70 & 0.97 & 0.70 \\
\bottomrule
\end{tabular}
}\label{tab:vidore_v1_jina}
\end{table*}

\begin{table*}[ht]
\centering
\caption{Experimental results of ColQwen on ViDoRe-V2.}
\resizebox{1.0\linewidth}{!}{
\begin{tabular}{llccllcccccccc}
\toprule
\textbf{Type} & \textbf{Methods} & \textbf{Adp.} & \textbf{Para} & \multicolumn{2}{c}{\textbf{all}} & \multicolumn{2}{c}{\textbf{esg}} & \multicolumn{2}{c}{\textbf{biomedical}} & \multicolumn{2}{c}{\textbf{economics}} & \multicolumn{2}{c}{\textbf{esg\_human}} \\
\cmidrule(lr){5-6} \cmidrule(lr){7-8} \cmidrule(lr){9-10} \cmidrule(lr){11-12} \cmidrule(lr){13-14}
& & & & ndcg & prun\% & ndcg & prun\% & ndcg & prun\% & ndcg & prun\% & ndcg & prun\% \\
\midrule
pruning & base &  & - & 0.55 & - & 0.51 & - & 0.58 & - & 0.49 & - & 0.61 & - \\
\rowcolor{lightgray} pruning & random &  & ratio=0.1 & 0.53 & 0.10 & 0.49 & 0.10 & 0.56 & 0.10 & 0.49 & 0.10 & 0.59 & 0.10 \\
pruning & random &  & ratio=0.3 & 0.50 & 0.30 & 0.45 & 0.30 & 0.55 & 0.30 & 0.49 & 0.30 & 0.50 & 0.30 \\
\rowcolor{lightgray} pruning & random &  & ratio=0.5 & 0.46 & 0.50 & 0.43 & 0.50 & 0.54 & 0.50 & 0.45 & 0.50 & 0.44 & 0.50 \\
pruning & random &  & ratio=0.7 & 0.45 & 0.70 & 0.39 & 0.70 & 0.53 & 0.70 & 0.47 & 0.70 & 0.40 & 0.70 \\
\rowcolor{lightgray} pruning & random &  & ratio=0.9 & 0.38 & 0.90 & 0.32 & 0.90 & 0.48 & 0.90 & 0.40 & 0.90 & 0.33 & 0.90 \\
\bottomrule
\end{tabular}
}\label{tab:vidore_v2_colqwen}
\end{table*}

\begin{table*}[ht]
\centering
\caption{Experimental results of ColNomic on ViDoRe-V2.}
\resizebox{1.0\linewidth}{!}{
\begin{tabular}{llccllcccccccc}
\toprule
\textbf{Type} & \textbf{Methods} & \textbf{Adp.} & \textbf{Para} & \multicolumn{2}{c}{\textbf{all}} & \multicolumn{2}{c}{\textbf{esg}} & \multicolumn{2}{c}{\textbf{biomedical}} & \multicolumn{2}{c}{\textbf{economics}} & \multicolumn{2}{c}{\textbf{esg\_human}} \\
\cmidrule(lr){5-6} \cmidrule(lr){7-8} \cmidrule(lr){9-10} \cmidrule(lr){11-12} \cmidrule(lr){13-14}
& & & & ndcg & prun\% & ndcg & prun\% & ndcg & prun\% & ndcg & prun\% & ndcg & prun\% \\
\midrule
pruning & base &  & - & 0.59 & - & 0.55 & - & 0.63 & - & 0.55 & - & 0.63 & - \\
\rowcolor{lightgray} pruning & random &  & ratio=0.1 & 0.59 & 0.10 & 0.54 & 0.10 & 0.62 & 0.10 & 0.55 & 0.10 & 0.65 & 0.10 \\
pruning & random &  & ratio=0.3 & 0.58 & 0.30 & 0.52 & 0.30 & 0.62 & 0.30 & 0.54 & 0.30 & 0.65 & 0.30 \\
\rowcolor{lightgray} pruning & random &  & ratio=0.5 & 0.57 & 0.50 & 0.50 & 0.50 & 0.62 & 0.50 & 0.55 & 0.50 & 0.61 & 0.50 \\
pruning & random &  & ratio=0.7 & 0.54 & 0.70 & 0.45 & 0.70 & 0.60 & 0.70 & 0.54 & 0.70 & 0.57 & 0.70 \\
\rowcolor{lightgray} pruning & random &  & ratio=0.9 & 0.48 & 0.90 & 0.39 & 0.90 & 0.55 & 0.90 & 0.47 & 0.90 & 0.50 & 0.90 \\
\bottomrule
\end{tabular}
}\label{tab:vidore_v2_colnomic}
\end{table*}

\begin{table*}[ht]
\centering
\caption{Experimental results of Jina-v4 on ViDoRe-V2.}
\resizebox{1.0\linewidth}{!}{
\begin{tabular}{llccllcccccccc}
\toprule
\textbf{Type} & \textbf{Methods} & \textbf{Adp.} & \textbf{Para} & \multicolumn{2}{c}{\textbf{all}} & \multicolumn{2}{c}{\textbf{esg}} & \multicolumn{2}{c}{\textbf{biomedical}} & \multicolumn{2}{c}{\textbf{economics}} & \multicolumn{2}{c}{\textbf{esg\_human}} \\
\cmidrule(lr){5-6} \cmidrule(lr){7-8} \cmidrule(lr){9-10} \cmidrule(lr){11-12} \cmidrule(lr){13-14}
& & & & ndcg & prun\% & ndcg & prun\% & ndcg & prun\% & ndcg & prun\% & ndcg & prun\% \\
\midrule
pruning & base &  & - & 0.56 & - & 0.50 & - & 0.60 & - & 0.54 & - & 0.62 & - \\
\rowcolor{lightgray} pruning & random &  & ratio=0.1 & 0.57 & 0.10 & 0.51 & 0.10 & 0.59 & 0.10 & 0.54 & 0.10 & 0.62 & 0.10 \\
pruning & random &  & ratio=0.3 & 0.56 & 0.30 & 0.49 & 0.30 & 0.59 & 0.30 & 0.54 & 0.30 & 0.62 & 0.30 \\
\rowcolor{lightgray} pruning & random &  & ratio=0.5 & 0.56 & 0.50 & 0.48 & 0.50 & 0.60 & 0.50 & 0.55 & 0.50 & 0.61 & 0.50 \\
pruning & random &  & ratio=0.7 & 0.52 & 0.70 & 0.44 & 0.70 & 0.58 & 0.70 & 0.51 & 0.70 & 0.56 & 0.70 \\
\rowcolor{lightgray} pruning & random &  & ratio=0.9 & 0.45 & 0.90 & 0.40 & 0.90 & 0.55 & 0.90 & 0.43 & 0.90 & 0.42 & 0.90 \\
\bottomrule
\end{tabular}
}\label{tab:vidore_v2_jina}
\end{table*}

\begin{table*}[ht]
\centering
\caption{Experimental results of ColQwen on JinaVDR.}
\resizebox{1.0\linewidth}{!}{
\begin{tabular}{llccllcccccccccccccccccccc}
\toprule
\textbf{Type} & \textbf{Methods} & \textbf{Adp.} & \textbf{Para} & \multicolumn{2}{c}{\textbf{all}} & \multicolumn{2}{c}{\textbf{arabic}} & \multicolumn{2}{c}{\textbf{japanese}} & \multicolumn{2}{c}{\textbf{russian}} & \multicolumn{2}{c}{\textbf{german}} & \multicolumn{2}{c}{\textbf{spanish}} & \multicolumn{2}{c}{\textbf{italian}} & \multicolumn{2}{c}{\textbf{dutch}} & \multicolumn{2}{c}{\textbf{hindi}} & \multicolumn{2}{c}{\textbf{chinese}} \\
\cmidrule(lr){5-6} \cmidrule(lr){7-8} \cmidrule(lr){9-10} \cmidrule(lr){11-12} \cmidrule(lr){13-14} \cmidrule(lr){15-16} \cmidrule(lr){17-18} \cmidrule(lr){19-20} \cmidrule(lr){21-22} \cmidrule(lr){23-24} 
& & & & ndcg & prun\% & ndcg & prun\% & ndcg & prun\% & ndcg & prun\% & ndcg & prun\% & ndcg & prun\% & ndcg & prun\% & ndcg & prun\% & ndcg & prun\% & ndcg & prun\% \\
\midrule
pruning & base &  & - & 0.54 & - & 0.46 & - & 0.74 & - & 0.74 & - & 0.35 & - & 0.56 & - & 0.61 & - & 0.44 & - & 0.07 & - & 0.90 & - \\
\rowcolor{lightgray} pruning & random &  & ratio=0.1 & 0.54 & 0.10 & 0.45 & 0.10 & 0.73 & 0.10 & 0.75 & 0.10 & 0.35 & 0.10 & 0.57 & 0.10 & 0.61 & 0.10 & 0.44 & 0.10 & 0.07 & 0.10 & 0.90 & 0.10 \\
pruning & random &  & ratio=0.3 & 0.54 & 0.30 & 0.46 & 0.30 & 0.72 & 0.30 & 0.75 & 0.30 & 0.34 & 0.30 & 0.57 & 0.30 & 0.61 & 0.30 & 0.44 & 0.30 & 0.07 & 0.30 & 0.89 & 0.30 \\
\rowcolor{lightgray} pruning & random &  & ratio=0.5 & 0.53 & 0.50 & 0.46 & 0.50 & 0.73 & 0.50 & 0.72 & 0.50 & 0.33 & 0.50 & 0.55 & 0.50 & 0.61 & 0.50 & 0.46 & 0.50 & 0.07 & 0.50 & 0.89 & 0.50 \\
pruning & random &  & ratio=0.7 & 0.52 & 0.70 & 0.44 & 0.70 & 0.69 & 0.70 & 0.74 & 0.70 & 0.31 & 0.70 & 0.53 & 0.70 & 0.59 & 0.70 & 0.43 & 0.70 & 0.06 & 0.70 & 0.88 & 0.70 \\
\rowcolor{lightgray} pruning & random &  & ratio=0.9 & 0.49 & 0.90 & 0.45 & 0.90 & 0.69 & 0.90 & 0.66 & 0.90 & 0.26 & 0.90 & 0.47 & 0.90 & 0.55 & 0.90 & 0.37 & 0.90 & 0.06 & 0.90 & 0.90 & 0.90 \\
\bottomrule
\end{tabular}
}\label{tab:jinavdr_colqwen}
\end{table*}

\begin{table*}[ht]
\centering
\caption{Experimental results of ColNomic on JinaVDR.}
\resizebox{1.0\linewidth}{!}{
\begin{tabular}{llccllcccccccccccccccccccc}
\toprule
\textbf{Type} & \textbf{Methods} & \textbf{Adp.} & \textbf{Para} & \multicolumn{2}{c}{\textbf{all}} & \multicolumn{2}{c}{\textbf{arabic}} & \multicolumn{2}{c}{\textbf{japanese}} & \multicolumn{2}{c}{\textbf{russian}} & \multicolumn{2}{c}{\textbf{german}} & \multicolumn{2}{c}{\textbf{spanish}} & \multicolumn{2}{c}{\textbf{italian}} & \multicolumn{2}{c}{\textbf{dutch}} & \multicolumn{2}{c}{\textbf{hindi}} & \multicolumn{2}{c}{\textbf{chinese}} \\
\cmidrule(lr){5-6} \cmidrule(lr){7-8} \cmidrule(lr){9-10} \cmidrule(lr){11-12} \cmidrule(lr){13-14} \cmidrule(lr){15-16} \cmidrule(lr){17-18} \cmidrule(lr){19-20} \cmidrule(lr){21-22} \cmidrule(lr){23-24} 
& & & & ndcg & prun\% & ndcg & prun\% & ndcg & prun\% & ndcg & prun\% & ndcg & prun\% & ndcg & prun\% & ndcg & prun\% & ndcg & prun\% & ndcg & prun\% & ndcg & prun\% \\
\midrule
pruning & base &  & - & 0.72 & - & 0.91 & - & 0.83 & - & 0.86 & - & 0.59 & - & 0.79 & - & 0.75 & - & 0.61 & - & 0.21 & - & 0.95 & - \\
\rowcolor{lightgray} pruning & random &  & ratio=0.1 & 0.72 & 0.10 & 0.92 & 0.10 & 0.83 & 0.10 & 0.87 & 0.10 & 0.59 & 0.10 & 0.79 & 0.10 & 0.74 & 0.10 & 0.62 & 0.10 & 0.21 & 0.10 & 0.95 & 0.10 \\
pruning & random &  & ratio=0.3 & 0.72 & 0.30 & 0.91 & 0.30 & 0.85 & 0.30 & 0.86 & 0.30 & 0.59 & 0.30 & 0.79 & 0.30 & 0.74 & 0.30 & 0.63 & 0.30 & 0.21 & 0.30 & 0.95 & 0.30 \\
\rowcolor{lightgray} pruning & random &  & ratio=0.5 & 0.72 & 0.50 & 0.90 & 0.50 & 0.86 & 0.50 & 0.87 & 0.50 & 0.58 & 0.50 & 0.77 & 0.50 & 0.74 & 0.50 & 0.61 & 0.50 & 0.21 & 0.50 & 0.94 & 0.50 \\
pruning & random &  & ratio=0.7 & 0.71 & 0.70 & 0.90 & 0.70 & 0.85 & 0.70 & 0.87 & 0.70 & 0.55 & 0.70 & 0.76 & 0.70 & 0.73 & 0.70 & 0.61 & 0.70 & 0.20 & 0.70 & 0.96 & 0.70 \\
\rowcolor{lightgray} pruning & random &  & ratio=0.9 & 0.67 & 0.90 & 0.84 & 0.90 & 0.82 & 0.90 & 0.85 & 0.90 & 0.47 & 0.90 & 0.67 & 0.90 & 0.69 & 0.90 & 0.59 & 0.90 & 0.18 & 0.90 & 0.94 & 0.90 \\
\bottomrule
\end{tabular}
}\label{tab:jinavdr_colnomic}
\end{table*}

\begin{table*}[ht]
\centering
\caption{Experimental results of Jina-v4 on JinaVDR.}
\resizebox{1.0\linewidth}{!}{
\begin{tabular}{llccllcccccccccccccccccccc}
\toprule
\textbf{Type} & \textbf{Methods} & \textbf{Adp.} & \textbf{Para} & \multicolumn{2}{c}{\textbf{all}} & \multicolumn{2}{c}{\textbf{arabic}} & \multicolumn{2}{c}{\textbf{japanese}} & \multicolumn{2}{c}{\textbf{russian}} & \multicolumn{2}{c}{\textbf{german}} & \multicolumn{2}{c}{\textbf{spanish}} & \multicolumn{2}{c}{\textbf{italian}} & \multicolumn{2}{c}{\textbf{dutch}} & \multicolumn{2}{c}{\textbf{hindi}} & \multicolumn{2}{c}{\textbf{chinese}} \\
\cmidrule(lr){5-6} \cmidrule(lr){7-8} \cmidrule(lr){9-10} \cmidrule(lr){11-12} \cmidrule(lr){13-14} \cmidrule(lr){15-16} \cmidrule(lr){17-18} \cmidrule(lr){19-20} \cmidrule(lr){21-22} \cmidrule(lr){23-24} 
& & & & ndcg & prun\% & ndcg & prun\% & ndcg & prun\% & ndcg & prun\% & ndcg & prun\% & ndcg & prun\% & ndcg & prun\% & ndcg & prun\% & ndcg & prun\% & ndcg & prun\% \\
\midrule
pruning & base &  & - & 0.66 & - & 0.92 & - & 0.74 & - & 0.80 & - & 0.54 & - & 0.66 & - & 0.68 & - & 0.47 & - & 0.17 & - & 0.97 & - \\
\rowcolor{lightgray} pruning & random &  & ratio=0.1 & 0.66 & 0.10 & 0.91 & 0.10 & 0.74 & 0.10 & 0.80 & 0.10 & 0.53 & 0.10 & 0.66 & 0.10 & 0.68 & 0.10 & 0.48 & 0.10 & 0.17 & 0.10 & 0.97 & 0.10 \\
pruning & random &  & ratio=0.3 & 0.66 & 0.30 & 0.91 & 0.30 & 0.76 & 0.30 & 0.80 & 0.30 & 0.54 & 0.30 & 0.66 & 0.30 & 0.68 & 0.30 & 0.48 & 0.30 & 0.17 & 0.30 & 0.96 & 0.30 \\
\rowcolor{lightgray} pruning & random &  & ratio=0.5 & 0.66 & 0.50 & 0.90 & 0.50 & 0.75 & 0.50 & 0.80 & 0.50 & 0.53 & 0.50 & 0.65 & 0.50 & 0.67 & 0.50 & 0.47 & 0.50 & 0.16 & 0.50 & 0.96 & 0.50 \\
pruning & random &  & ratio=0.7 & 0.64 & 0.70 & 0.89 & 0.70 & 0.75 & 0.70 & 0.79 & 0.70 & 0.52 & 0.70 & 0.61 & 0.70 & 0.66 & 0.70 & 0.46 & 0.70 & 0.16 & 0.70 & 0.96 & 0.70 \\
\rowcolor{lightgray} pruning & random &  & ratio=0.9 & 0.62 & 0.90 & 0.87 & 0.90 & 0.74 & 0.90 & 0.79 & 0.90 & 0.48 & 0.90 & 0.55 & 0.90 & 0.61 & 0.90 & 0.42 & 0.90 & 0.14 & 0.90 & 0.95 & 0.90 \\
\bottomrule
\end{tabular}
}\label{tab:jinavdr_jina}
\end{table*}

\begin{table*}[ht]
\centering
\caption{Experimental results of ColQwen on REAL-MM-RAG.}
\resizebox{1.0\linewidth}{!}{
\begin{tabular}{llccllcccccccc}
\toprule
\textbf{Type} & \textbf{Methods} & \textbf{Adp.} & \textbf{Para} & \multicolumn{2}{c}{\textbf{all}} & \multicolumn{2}{c}{\textbf{finreport}} & \multicolumn{2}{c}{\textbf{finslides}} & \multicolumn{2}{c}{\textbf{techreport}} & \multicolumn{2}{c}{\textbf{techslides}} \\
\cmidrule(lr){5-6} \cmidrule(lr){7-8} \cmidrule(lr){9-10} \cmidrule(lr){11-12} \cmidrule(lr){13-14}
& & & & ndcg & prun\% & ndcg & prun\% & ndcg & prun\% & ndcg & prun\% & ndcg & prun\% \\
\midrule
pruning & base &  & - & 0.72 & - & 0.55 & - & 0.62 & - & 0.81 & - & 0.90 & - \\
\rowcolor{lightgray} pruning & random &  & ratio=0.1 & 0.72 & 0.10 & 0.55 & 0.10 & 0.62 & 0.10 & 0.81 & 0.10 & 0.90 & 0.10 \\
pruning & random &  & ratio=0.3 & 0.72 & 0.30 & 0.55 & 0.30 & 0.60 & 0.30 & 0.80 & 0.30 & 0.91 & 0.30 \\
\rowcolor{lightgray} pruning & random &  & ratio=0.5 & 0.70 & 0.50 & 0.54 & 0.50 & 0.58 & 0.50 & 0.79 & 0.50 & 0.90 & 0.50 \\
pruning & random &  & ratio=0.7 & 0.68 & 0.70 & 0.50 & 0.70 & 0.53 & 0.70 & 0.78 & 0.70 & 0.90 & 0.70 \\
\rowcolor{lightgray} pruning & random &  & ratio=0.9 & 0.58 & 0.90 & 0.41 & 0.90 & 0.40 & 0.90 & 0.67 & 0.90 & 0.85 & 0.90 \\
\bottomrule
\end{tabular}
}\label{tab:realmmrag_colqwen}
\end{table*}

\begin{table*}[ht]
\centering
\caption{Experimental results of ColNomic on REAL-MM-RAG.}
\resizebox{1.0\linewidth}{!}{
\begin{tabular}{llccllcccccccc}
\toprule
\textbf{Type} & \textbf{Methods} & \textbf{Adp.} & \textbf{Para} & \multicolumn{2}{c}{\textbf{all}} & \multicolumn{2}{c}{\textbf{finreport}} & \multicolumn{2}{c}{\textbf{finslides}} & \multicolumn{2}{c}{\textbf{techreport}} & \multicolumn{2}{c}{\textbf{techslides}} \\
\cmidrule(lr){5-6} \cmidrule(lr){7-8} \cmidrule(lr){9-10} \cmidrule(lr){11-12} \cmidrule(lr){13-14}
& & & & ndcg & prun\% & ndcg & prun\% & ndcg & prun\% & ndcg & prun\% & ndcg & prun\% \\
\midrule
pruning & base &  & - & 0.78 & - & 0.63 & - & 0.73 & - & 0.83 & - & 0.92 & - \\
\rowcolor{lightgray} pruning & random &  & ratio=0.1 & 0.84 & 0.10 & 0.76 & 0.10 & 0.80 & 0.10 & 0.88 & 0.10 & 0.91 & 0.10 \\
pruning & random &  & ratio=0.3 & 0.83 & 0.30 & 0.76 & 0.30 & 0.79 & 0.30 & 0.87 & 0.30 & 0.92 & 0.30 \\
\rowcolor{lightgray} pruning & random &  & ratio=0.5 & 0.81 & 0.50 & 0.72 & 0.50 & 0.75 & 0.50 & 0.86 & 0.50 & 0.91 & 0.50 \\
pruning & random &  & ratio=0.7 & 0.77 & 0.70 & 0.67 & 0.70 & 0.67 & 0.70 & 0.82 & 0.70 & 0.90 & 0.70 \\
\rowcolor{lightgray} pruning & random &  & ratio=0.9 & 0.63 & 0.90 & 0.50 & 0.90 & 0.48 & 0.90 & 0.68 & 0.90 & 0.85 & 0.90 \\
\bottomrule
\end{tabular}
}\label{tab:realmmrag_colnomic}
\end{table*}

\begin{table*}[ht]
\centering
\caption{Experimental results of Jina-v4 on REAL-MM-RAG.}
\resizebox{1.0\linewidth}{!}{
\begin{tabular}{llccllcccccccc}
\toprule
\textbf{Type} & \textbf{Methods} & \textbf{Adp.} & \textbf{Para} & \multicolumn{2}{c}{\textbf{overall}} & \multicolumn{2}{c}{\textbf{finreport}} & \multicolumn{2}{c}{\textbf{finslides}} & \multicolumn{2}{c}{\textbf{techreport}} & \multicolumn{2}{c}{\textbf{techslides}} \\
\cmidrule(lr){5-6} \cmidrule(lr){7-8} \cmidrule(lr){9-10} \cmidrule(lr){11-12} \cmidrule(lr){13-14}
& & & & ndcg & prun\% & ndcg & prun\% & ndcg & prun\% & ndcg & prun\% & ndcg & prun\% \\
\midrule
pruning & base &  & - & 0.79 & - & 0.69 & - & 0.73 & - & 0.82 & - & 0.91 & - \\
\rowcolor{lightgray} pruning & random &  & ratio=0.1 & 0.79 & 0.10 & 0.69 & 0.10 & 0.73 & 0.10 & 0.82 & 0.10 & 0.91 & 0.10 \\
pruning & random &  & ratio=0.3 & 0.79 & 0.30 & 0.69 & 0.30 & 0.72 & 0.30 & 0.82 & 0.30 & 0.92 & 0.30 \\
\rowcolor{lightgray} pruning & random &  & ratio=0.5 & 0.78 & 0.50 & 0.67 & 0.50 & 0.72 & 0.50 & 0.81 & 0.50 & 0.91 & 0.50 \\
pruning & random &  & ratio=0.7 & 0.74 & 0.70 & 0.63 & 0.70 & 0.65 & 0.70 & 0.79 & 0.70 & 0.90 & 0.70 \\
\rowcolor{lightgray} pruning & random &  & ratio=0.9 & 0.64 & 0.90 & 0.51 & 0.90 & 0.47 & 0.90 & 0.70 & 0.90 & 0.87 & 0.90 \\
\bottomrule
\end{tabular}
}\label{tab:realmmrag_jina}
\end{table*}

\begin{table*}[ht]
\centering
\caption{Experimental results of ColQwen on ViDoSeek.}
\resizebox{0.5\linewidth}{!}{
\begin{tabular}{llccll}
\toprule
\textbf{Type} & \textbf{Methods} & \textbf{Adp.} & \textbf{Para} & \multicolumn{2}{c}{\textbf{overall}} \\
\cmidrule(lr){5-6}
& & & & ndcg@5 & pruning\% \\
\midrule
pruning & base &  & - & 0.82 & - \\
\rowcolor{lightgray} pruning & random &  & ratio=0.1 & 0.82 & 0.10 \\
pruning & random &  & ratio=0.3 & 0.82 & 0.30 \\
\rowcolor{lightgray} pruning & random &  & ratio=0.5 & 0.82 & 0.50 \\
pruning & random &  & ratio=0.7 & 0.82 & 0.70 \\
\rowcolor{lightgray} pruning & random &  & ratio=0.9 & 0.81 & 0.90 \\
\bottomrule
\end{tabular}
}\label{tab:vidoseek_colqwen}
\end{table*}

\begin{table*}[ht]
\centering
\caption{Experimental results of ColNomic on ViDoSeek.}
\resizebox{0.5\linewidth}{!}{
\begin{tabular}{llccll}
\toprule
\textbf{Type} & \textbf{Methods} & \textbf{Adp.} & \textbf{Para} & \multicolumn{2}{c}{\textbf{overall}} \\
\cmidrule(lr){5-6}
& & & & ndcg@5 & pruning\% \\
\midrule
pruning & base &  & - & 0.84 & - \\
\rowcolor{lightgray} pruning & random &  & ratio=0.1 & 0.83 & 0.10 \\
pruning & random &  & ratio=0.3 & 0.83 & 0.30 \\
\rowcolor{lightgray} pruning & random &  & ratio=0.5 & 0.83 & 0.50 \\
pruning & random &  & ratio=0.7 & 0.83 & 0.70 \\
\rowcolor{lightgray} pruning & random &  & ratio=0.9 & 0.81 & 0.90 \\
\bottomrule
\end{tabular}
}\label{tab:vidoseek_colnomic}
\end{table*}

\begin{table*}[ht]
\centering
\caption{Experimental results of Jina-v4 on ViDoSeek.}
\resizebox{0.5\linewidth}{!}{
\begin{tabular}{llccll}
\toprule
\textbf{Type} & \textbf{Methods} & \textbf{Adp.} & \textbf{Para} & \multicolumn{2}{c}{\textbf{overall}} \\
\cmidrule(lr){5-6}
& & & & ndcg@5 & pruning\% \\
\midrule
pruning & base &  & - & 0.81 & - \\
\rowcolor{lightgray} pruning & random &  & ratio=0.1 & 0.81 & 0.10 \\
pruning & random &  & ratio=0.3 & 0.81 & 0.30 \\
\rowcolor{lightgray} pruning & random &  & ratio=0.5 & 0.81 & 0.50 \\
pruning & random &  & ratio=0.7 & 0.81 & 0.70 \\
\rowcolor{lightgray} pruning & random &  & ratio=0.9 & 0.79 & 0.90 \\
\bottomrule
\end{tabular}
}\label{tab:vidoseek_jina}
\end{table*}

\begin{table*}[ht]
\centering
\caption{Experimental results of ColQwen on MMLongBench-Doc.}
\resizebox{0.5\linewidth}{!}{
\begin{tabular}{llccll}
\toprule
\textbf{Type} & \textbf{Methods} & \textbf{Adp.} & \textbf{Para} & \multicolumn{2}{c}{\textbf{overall}} \\
\cmidrule(lr){5-6}
& & & & ndcg@5 & pruning\% \\
\midrule
pruning & base &  & - & 0.54 & - \\
\rowcolor{lightgray} pruning & random &  & ratio=0.1 & 0.54 & 0.10 \\
pruning & random &  & ratio=0.3 & 0.53 & 0.30 \\
\rowcolor{lightgray} pruning & random &  & ratio=0.5 & 0.53 & 0.50 \\
pruning & random &  & ratio=0.7 & 0.51 & 0.70 \\
\rowcolor{lightgray} pruning & random &  & ratio=0.9 & 0.47 & 0.90 \\
\bottomrule
\end{tabular}
}\label{tab:mmlongbench_colqwen}
\end{table*}

\begin{table*}[ht]
\centering
\caption{Experimental results of ColNomic on MMLongBench-Doc.}
\resizebox{0.5\linewidth}{!}{
\begin{tabular}{llccll}
\toprule
\textbf{Type} & \textbf{Methods} & \textbf{Adp.} & \textbf{Para} & \multicolumn{2}{c}{\textbf{overall}} \\
\cmidrule(lr){5-6}
& & & & ndcg@5 & pruning\% \\
\midrule
pruning & base &  & - & 0.57 & - \\
\rowcolor{lightgray} pruning & random &  & ratio=0.1 & 0.57 & 0.10 \\
pruning & random &  & ratio=0.3 & 0.55 & 0.30 \\
\rowcolor{lightgray} pruning & random &  & ratio=0.5 & 0.54 & 0.50 \\
pruning & random &  & ratio=0.7 & 0.52 & 0.70 \\
\rowcolor{lightgray} pruning & random &  & ratio=0.9 & 0.46 & 0.90 \\
\bottomrule
\end{tabular}
}\label{tab:mmlongbench_colnomic}
\end{table*}

\begin{table*}[ht]
\centering
\caption{Experimental results of Jina-v4 on MMLongBench-Doc.}
\resizebox{0.5\linewidth}{!}{
\begin{tabular}{llccll}
\toprule
\textbf{Type} & \textbf{Methods} & \textbf{Adp.} & \textbf{Para} & \multicolumn{2}{c}{\textbf{overall}} \\
\cmidrule(lr){5-6}
& & & & ndcg@5 & pruning\% \\
\midrule
pruning & base &  & - & 0.54 & - \\
\rowcolor{lightgray} pruning & random &  & ratio=0.1 & 0.54 & 0.10 \\
pruning & random &  & ratio=0.3 & 0.54 & 0.30 \\
\rowcolor{lightgray} pruning & random &  & ratio=0.5 & 0.53 & 0.50 \\
pruning & random &  & ratio=0.7 & 0.52 & 0.70 \\
\rowcolor{lightgray} pruning & random &  & ratio=0.9 & 0.45 & 0.90 \\
\bottomrule
\end{tabular}
}\label{tab:mmlongbench_jina}
\end{table*}

\section{Usage of AI Assistant}
\label{app:usage_ai_assistant}

We utilized Google's Gemini 2.5 Pro exclusively for the purpose of polishing the manuscript's language and improving its readability. All other aspects of this research, including the core conceptualization of the framework, experimental design, implementation, and the analysis and interpretation of results, were conducted entirely by the authors.